\documentclass[10pt,twocolumn]{IEEEtran}

\usepackage[utf8]{inputenc}
\usepackage[T1]{fontenc}
\usepackage{microtype}

\usepackage{amsmath,amssymb,amsthm}
\usepackage{mathtools}
\usepackage{graphicx}
\usepackage{booktabs}
\usepackage{array}

\usepackage{enumitem}
\usepackage{xspace}
\usepackage[table,dvipsnames]{xcolor}
\usepackage{cuted}
\usepackage{capt-of}
\usepackage{tikz}
\usetikzlibrary{calc,spy}
\usepackage{hyperref}
\hypersetup{colorlinks=true, allcolors=Blue, citecolor=Blue, urlcolor=Blue, linkcolor=Blue}
\usepackage[capitalize]{cleveref}
\definecolor{rowhl}{HTML}{E3ECF7}
\definecolor{rowtot}{HTML}{EFEFEF}
\definecolor{oursmint}{HTML}{E2F1EA}
\definecolor{losssalmon}{HTML}{F9E8E4}
\definecolor{figok}{HTML}{1A7A3C}
\definecolor{figbad}{HTML}{C03028}

\newlength{\teaserpanelw}
\newcommand{\panel}[2][]{%
  \begin{tikzpicture}[baseline=(img.north)]
    \node[draw, line width=0.8pt, inner sep=0pt, outer sep=0pt] (img)
      {\includegraphics[width=\dimexpr\teaserpanelw-0.8pt\relax,#1]{#2}};
  \end{tikzpicture}%
}

\newcommand{\firmgrasp}{{{\scshape Firm}Grasp}}
\newcommand{\Cfg}{\mathcal{Q}}                     
\newcommand{\Gset}{\mathcal{G}}                    
\newcommand{\Sset}{\mathcal{C}}                    
\newcommand{\Real}{\mathbb{R}}
\newcommand{\Realnn}{\mathbb{R}_{\ge 0}}
\newcommand{\SEthree}{\mathrm{SE}(3)}

\newcommand{\qpre}[1]{q^{\mathrm{pre}}_{#1}}
\newcommand{\qgrasp}[1]{q^{\mathrm{grasp}}_{#1}}
\newcommand{\lbar}{\bar{\ell}^{*}}
\newcommand{\epsb}{\varepsilon^{(\beta)}}
\newcommand{\Int}{\operatorname{int}}
\newcommand{\conv}{\operatorname{conv}}
\DeclareMathOperator*{\argmin}{arg\,min}

\newtheorem{definition}{Definition}[section]
\newtheorem{lemma}{Lemma}[section]
\newtheorem{proposition}{Proposition}[section]
\newtheorem{theorem}{Theorem}[section]
\newtheorem{corollary}{Corollary}[section]
\newtheorem{assumption}{Assumption}[section]
\newtheorem{problem}{Problem}
\theoremstyle{remark}
\newtheorem{remark}{Remark}[section]

\newlength{\panelw}

\newsavebox{\gdfleftcol}
\newcommand{\cornertag}[2]{%
  \begin{tikzpicture}[inner sep=0pt, outer sep=0pt]
    \node[inner sep=0pt] (cornertagimg) {#1};
    \node[anchor=north west, fill=lightgray!30, draw=black, line width=0.0pt,
          inner sep=1.5pt, minimum width=1.1em, minimum height=0.9em,
          font=\small, align=left] at (cornertagimg.north west) {#2};
  \end{tikzpicture}%
}

\title{\Large\textbf{Grasp Execution Without a Planner}:\\
Configuration-Space Grasp Distance Fields with Certified Safety \& Guaranteed Quality}

\author{Clinton Enwerem\(^{1}\), John S. Baras\(^{1}\), and Calin Belta\(^{1}\)%
\thanks{\(^{1}\)The authors are with the Department of Electrical and Computer Engineering and the Institute for Systems Research (ISR), University of Maryland, College Park, MD, USA. Emails:
{\ttfamily\small \char`\{enwerem, baras, calin\char`\}@umd.edu}.}%
}

\begin{document}
\maketitle
\begin{abstract}
Mainstream plan-then-track approaches to multifingered grasp execution entail selecting a grasp, planning a collision-free trajectory, and tracking the resulting trajectory via a feedback controller. Pose-estimation error during execution or scene motion can invalidate this open-loop commitment and trigger replanning. We thus present Grasp Distance Fields (GDFs), smooth softmin distances to finite sets of arm-hand grasp configurations. Using their negative gradients as feedback, we jointly select and execute grasps without planning a trajectory. A CBF-CLF quadratic program (QP) enforces self-collision, workspace, object, and obstacle-clearance constraints, while its CLF slack quantifies obstruction of task progress. We bound the softmin approximation error by $\log N/\rho$ and prove forward invariance of the filtered safe set. To handle changes in contact topology, we combine a hysteretic contact-mode transition with a wrench-quality CBF that limits degradation of the realized force-closure margin relative to hold onset. Using our method, a fixed-base manipulator and a Unitree G1 equipped with the same underactuated hand grasp and lift 46 of 50 test objects amid clutter and moving obstacles. The realized grasps also retain a median 94\% of their synthesized quality margin, and each QP solve requires 0.09 ms within a 20 ms control interval. Project page: \texttt{\url{www.clintonenwerem.com/gdf}}.
\end{abstract}


\section{Introduction}
\label{sec:intro}
Multifingered grasp pipelines typically generate candidates \cite{sundermeyer2021contact,wang_dexgraspnet_2023}, select one candidate, plan a collision-free trajectory, and track it through approach, closure, and object relocation~\cite{miller_graspit_2004}. Object-pose error and scene motion can invalidate the selected trajectory, and replanning latency motivates reactive execution \cite{morrison2018ggcnn,singletary2022safety}.
\begin{figure}[t]
  \centering
    \setlength{\fboxsep}{0pt}
    \setlength{\fboxrule}{0.5pt}
    \sbox{\gdfleftcol}{%
    \begin{minipage}[b]{0.485\columnwidth}
      \begin{minipage}[t]{\columnwidth}
        \centering
        \fbox{\cornertag{%
          \begin{tikzpicture}[
              spy using outlines={rectangle, magnification=3,
                  width=0.95cm, height=0.85cm, connect spies}
          ]
            \node[anchor=south west, inner sep=0] (imga) at (0,0) {%
              \includegraphics[width=.84\linewidth, trim=0pt 360pt 0pt 110pt, clip]%
              {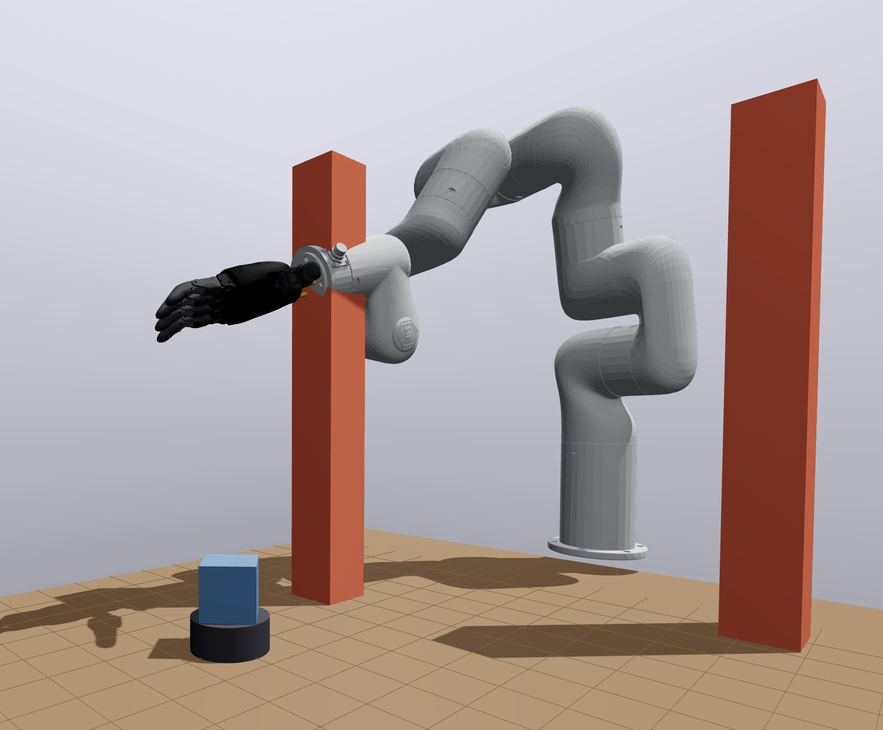}%
            };
            \begin{scope}[x={($(imga.south east)-(imga.south west)$)},
                          y={($(imga.north west)-(imga.south west)$)},
                          shift={(imga.south west)}]
              \coordinate (spypt) at (0.355,0.3250);
            \end{scope}
            \spy [black, thick] on (spypt)
                 in node [anchor=north east, inner sep=0, xshift=-5pt, yshift=-2.5pt] at (imga.north east);
          \end{tikzpicture}%
        }{\textsf{Collision}}}\\[-0.15em]
        {\footnotesize (a) \textsc{Nominal Command}}\\[-0.15em]
        {\footnotesize \(\min_t h_{\mathrm{obs}} \!=\! \textcolor{figbad}{\mathsf{-29.8\,\mathrm{mm}}}\)}
    \end{minipage}\\[.51em]
    \begin{minipage}[b]{\columnwidth}
        \centering
        \fbox{\cornertag{{\includegraphics[width=.84\linewidth, trim=0pt 60pt 0pt 90pt, clip]{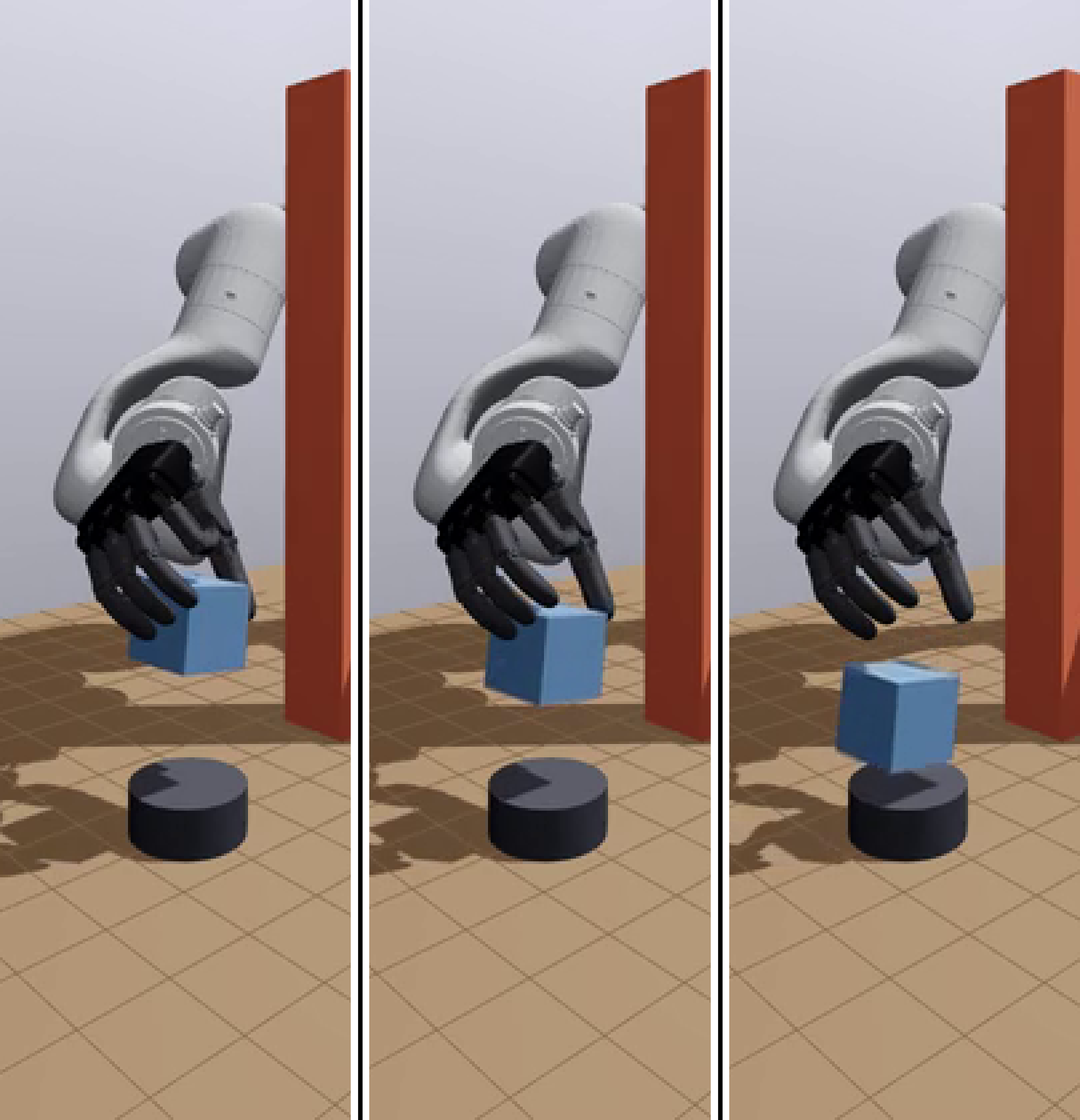}}}{\textsf{Certificate Failure}}}\\[-0.15em]
        {\footnotesize (b) \textsc{{Quality-Neutral QP}}}\\[-0.15em]
        \resizebox{\linewidth}{!}{\footnotesize \({\epsb_{\mathrm{exec}}}\!=\!\textcolor{figbad}{\mathsf{-0.32}}\),\thinspace \({\epsb_{\mathrm{desc}}}\!=\!\textcolor{figok}{\mathsf{1.83{\times}10^{-3}}}\)}
    \end{minipage}
  \end{minipage}}%
  \usebox{\gdfleftcol}\hfill
    \begin{minipage}[b][\dimexpr\ht\gdfleftcol+\dp\gdfleftcol\relax][s]{0.49\columnwidth}
        \centering
        {\footnotesize (c) \textsc{Grasp Distance Field (Ours)}}\\[0.75em]
        \fbox{\cornertag{\includegraphics[width=\dimexpr\linewidth-2\fboxrule\relax]{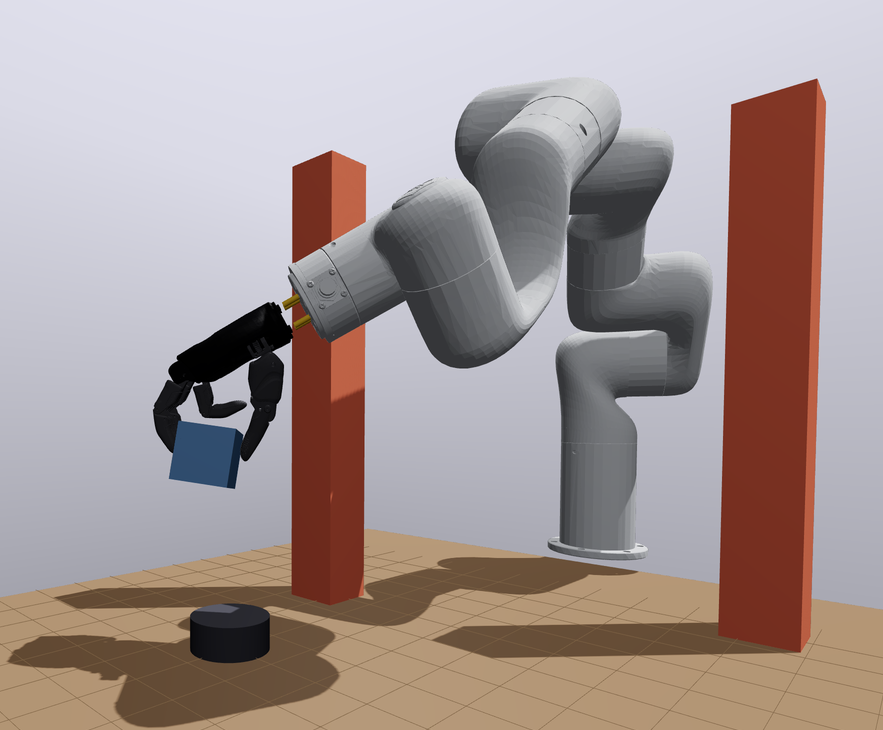}}{\textsf{Task Success}}}\\[-0.15em]
        {\footnotesize \(\min_t h_{\mathrm{obs}} \!=\! \textcolor{figok}{\mathsf{+17.4\,\mathrm{mm}}}\), \(\epsb_{\mathrm{exec}} \!\ge\! \textcolor{figok}{\mathsf{3.80\times10^{-3}}}\)}\par
    \end{minipage}
    \caption{\textbf{Safe Grasp Execution with Quality Sublevel-Set Forward Invariance.} In (a), the trajectory under the nominal command reaches the object but drives the wrist into an obstacle. In (c), our filtered closed loop completes the task with every CBF positive. In (b), the quality-neutral program without a closure certificate grasps and lifts an object that slips almost immediately under gravity (columns from left to right capture task evolution post-lift). With our wrench-quality constraint (\cref{sec:filter}) enforced, the program turns infeasible during the hold and the trial stops before the lift. In the sub-captions, \(h_{\mathrm{obs}}\) is our obstacle CBF (\cref{sec:filter}), and \(\smash{\epsb_{\mathrm{desc}}}\) and \(\smash{\epsb_{\mathrm{exec}}}\) are the risk-adjusted grasp quality margins \cite{enwerem2026firmgrasp} of a stored grasp and of its realized contact set. \textcolor{figbad}{Red} numbers mark CBF or certificate violations, and \textcolor{figok}{green} numbers mark certified margins. We report each figure's associated evaluation in \cref{sec:eval}, and \cref{fig:epsbtrace} plots the per-step margins.}
    \label{fig:bslnvsgdf}
\end{figure}
By contrast, we omit the replanning step entirely and represent the grasp execution target as a set of grasp configurations. We then define a smooth distance field over this set in the configuration space of the coupled arm-hand system, and specify a stationary feedback law on the resulting field. As such, we sidestep the discrete grasp candidate selection step, since the field's negative gradient points toward whichever candidate lies nearest in joint space from the current configuration. 

Since execution must also satisfy several safety constraints alongside convergence to the target set, we therefore project the field-based nominal input to a safe set characterized by a Control Barrier Function (CBF)~\cite{ames2019cbf} defined over critical safety variables, including self-collision, workspace, object, and obstacle clearance. We solve the resulting quadratic program once per control step. Since the QP-based synthesis framework is amenable to dynamic constraints~\cite{ames2019cbf,ames2017cbfqp}, our controller handles workspace changes or perturbed states by evaluating the same state-to-velocity map. We evaluate our controller in three workspace settings with two robotic embodiments (\cref{fig:teaser}), and \cref{fig:bslnvsgdf} contrasts the closed-loop system under the safety filter against the nominal command on a reach-avoid-stay grasping task.

\subsection{Related Work}
\label{sec:related}

Operational-space potential field methods place an attractive potential at the goal and repulsive potentials on obstacles, evaluated at each control step \cite{khatib1986realtime}. Their failure mode, a local minimum away from the goal, has motivated modifications since. The CBF formulation retains the reactive character while converting obstacle avoidance from a penalty into a constraint with a forward-invariance guarantee. Every potential field induces a barrier function while the converse fails, and the CBF therefore strictly generalizes the potential field \cite{singletary2021comparative,ames2019cbf}.

Beyond reactive navigation, CBF methods extend to grasping and manipulation. The authors of~\cite{shawcortez2019cbf} apply them to grasp tracking through controllers that certify no slip and no roll-off, and~\cite{kim2025tactilegrasp} enforces contact-force and closure constraints from measured contacts during the grasp itself. On robotic manipulators, the same filters operate at every control step in clutter, intervening around dense dynamic obstacles in a cooking cell \cite{singletary2022safety}. Perception-built safety functions omit the hand-designed distance model entirely, solving a Poisson problem on occupancy data so that the CBF value comes from measured geometry \cite{bahati2025poisson,wilkinson2026fullbody}, with predictive humanoid implementations showing the construction at whole-body scale \cite{bena2025geometry}. Since a multifingered grasp is not a point, extending such field-based methods to grasping requires the definition of a grasp target.
\begin{figure*}[t]
\centering
\setlength{\tabcolsep}{1.5pt}%
\begin{tabular}{@{}cc@{\hspace{4pt}}cc@{\hspace{4pt}}cc@{\hspace{4pt}}cc@{}}
{\footnotesize Initial} & {\footnotesize Final} &
{\footnotesize Initial} & {\footnotesize Final} &
{\footnotesize Initial} & {\footnotesize Final} &
{\footnotesize Initial} & {\footnotesize Final} \\[-6pt]
\panel[trim=200pt 27pt 0pt 20pt, clip]{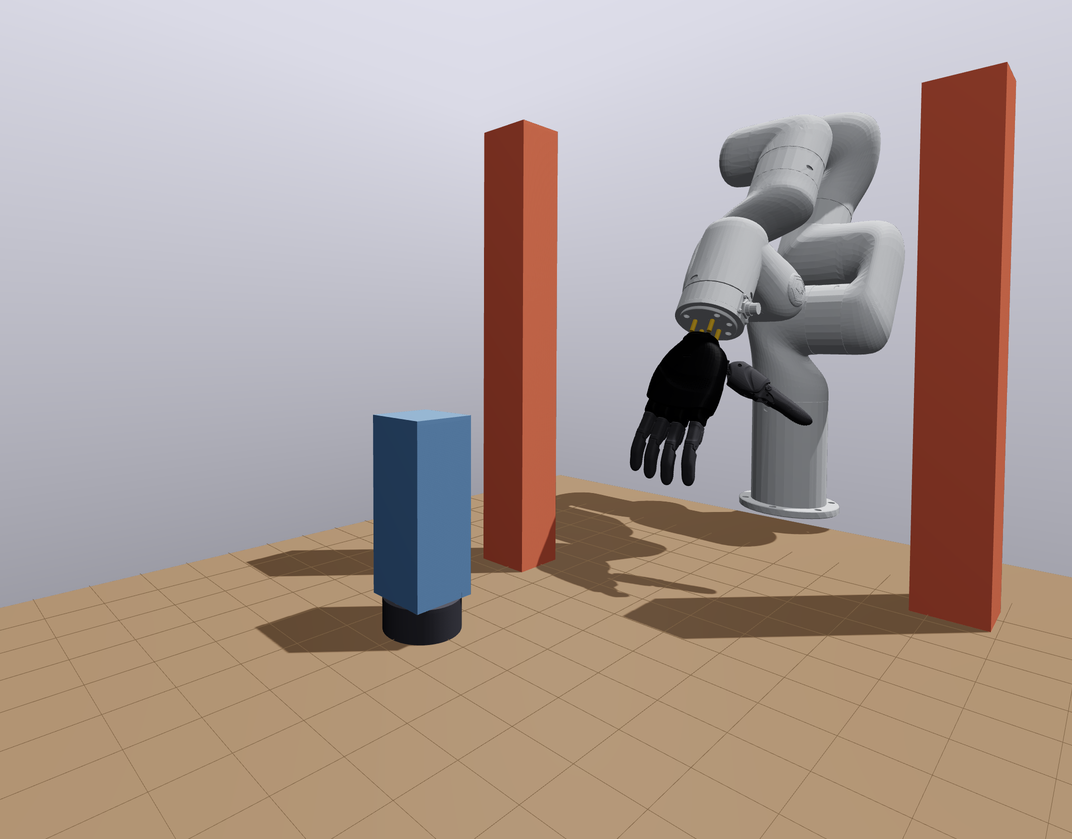} &
\panel[trim=200pt 27pt 0pt 20pt, clip]{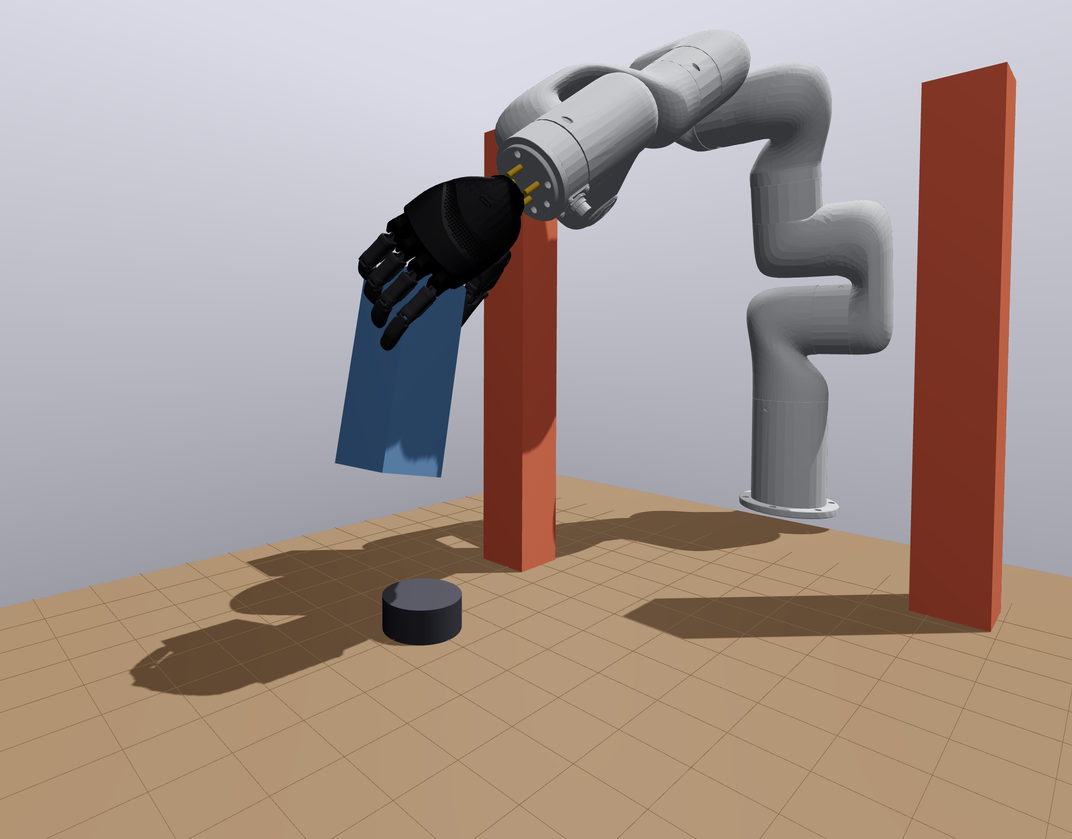} &
\panel[trim=400pt 200pt 45pt 50pt, clip]{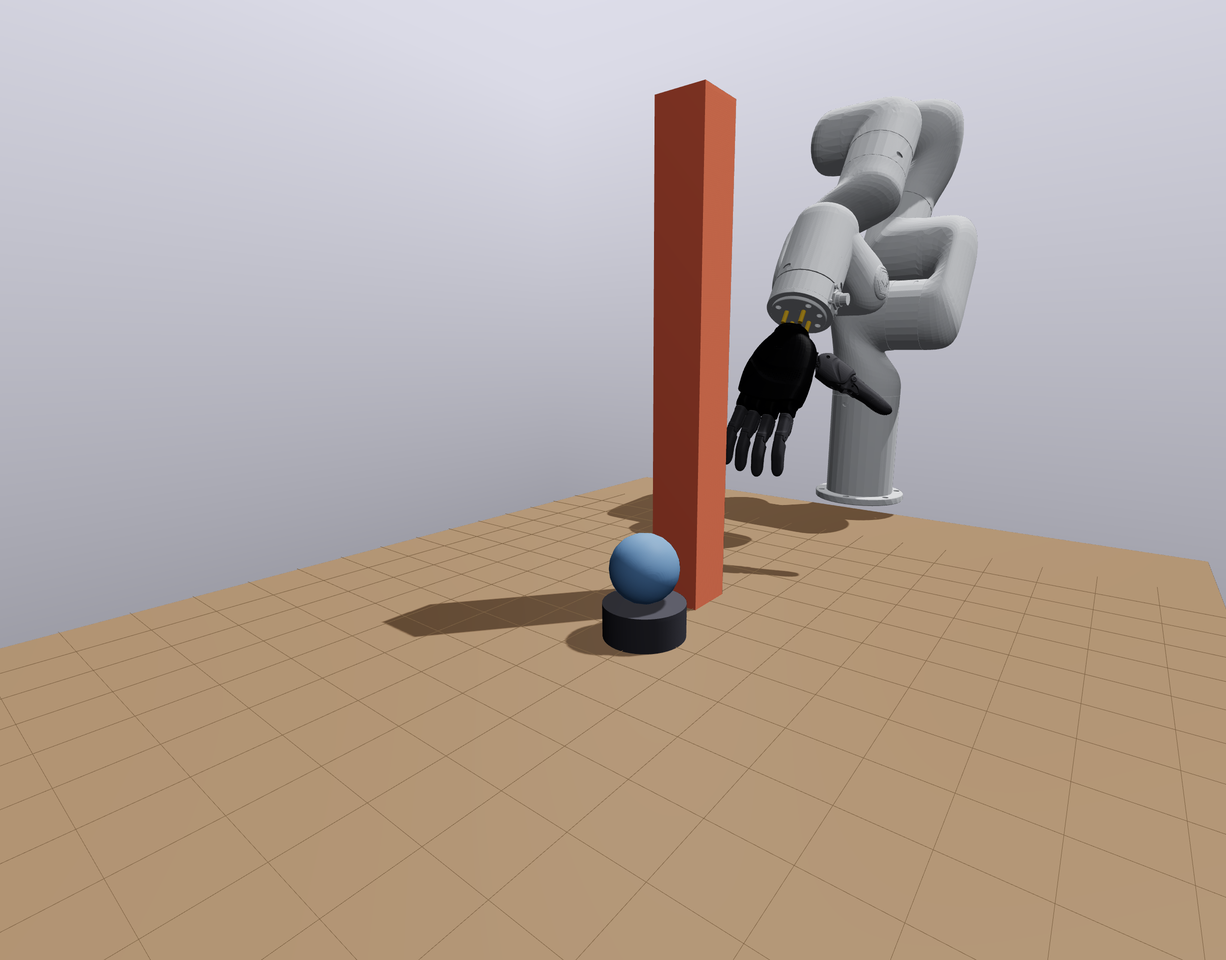} &
\panel[trim=400pt 200pt 45pt 50pt, clip]{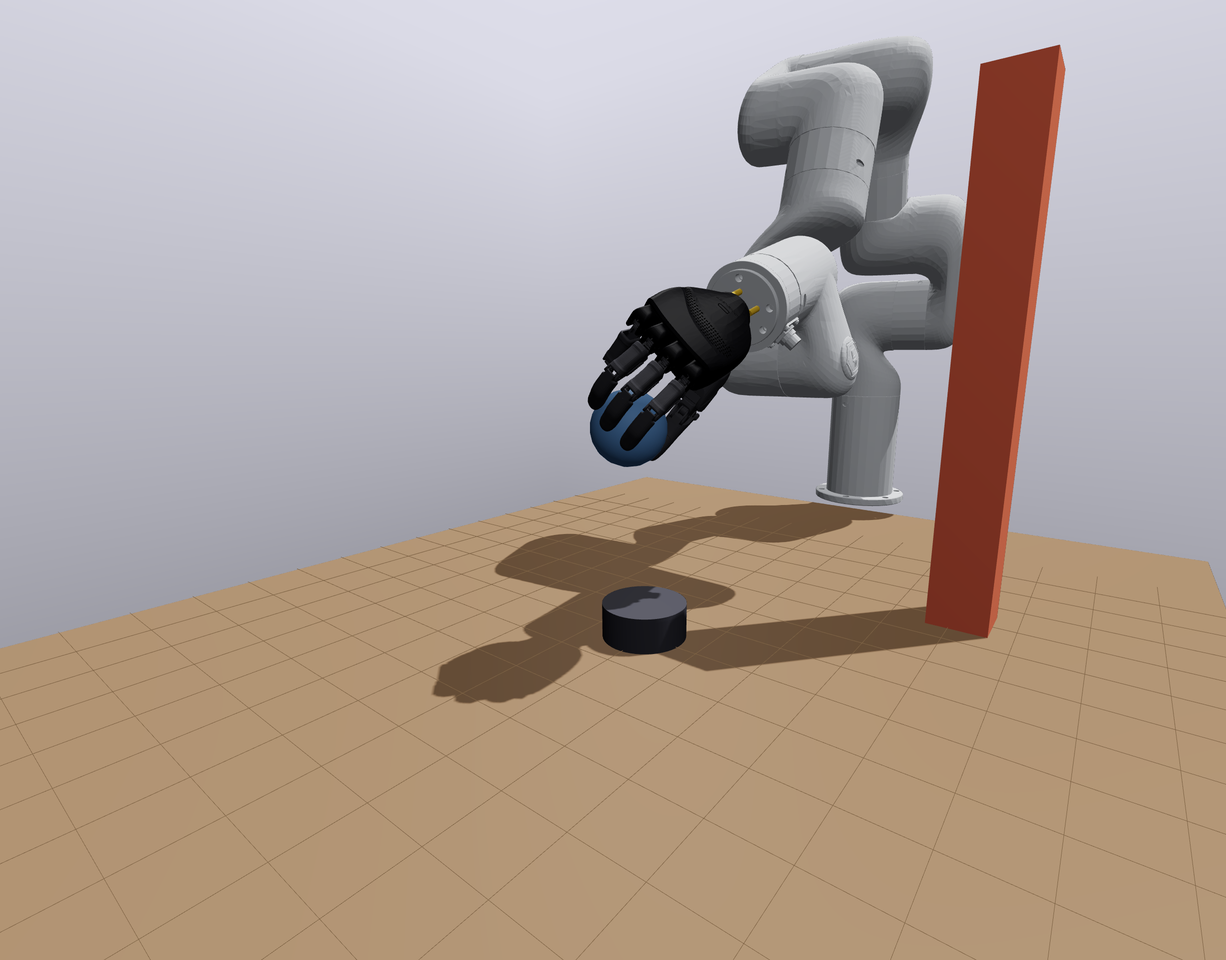} &
\panel[trim=0pt 0pt 0pt 75pt, clip]{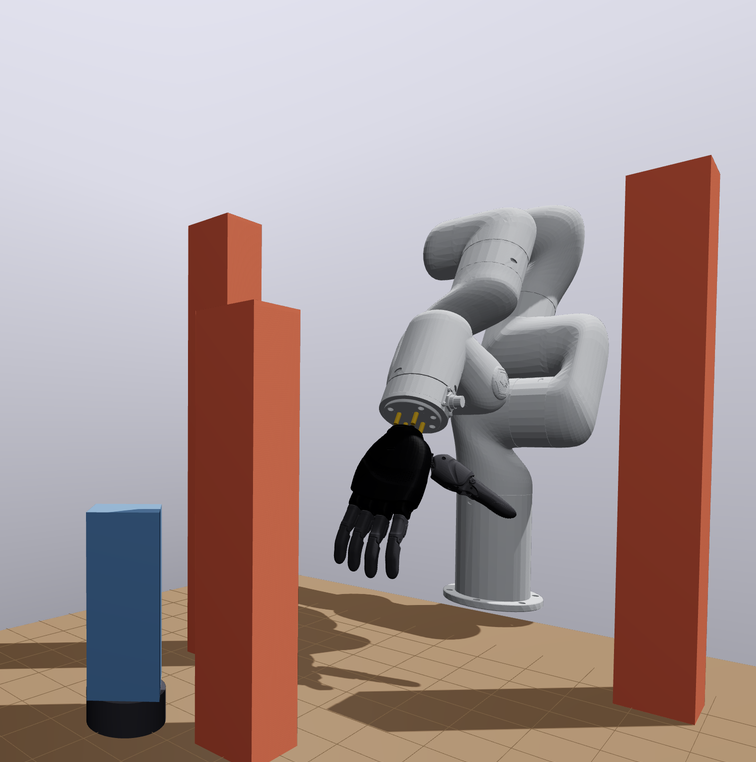}{} &
\panel[trim=0pt 0pt 0pt 75pt, clip]{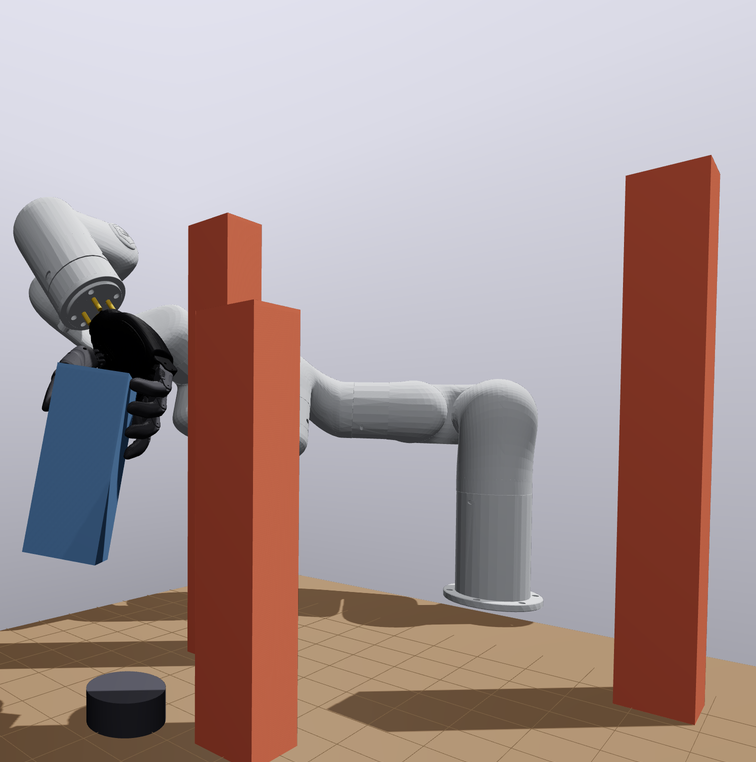}{} &
\panel[trim=150pt 200pt 294pt 50pt, clip]{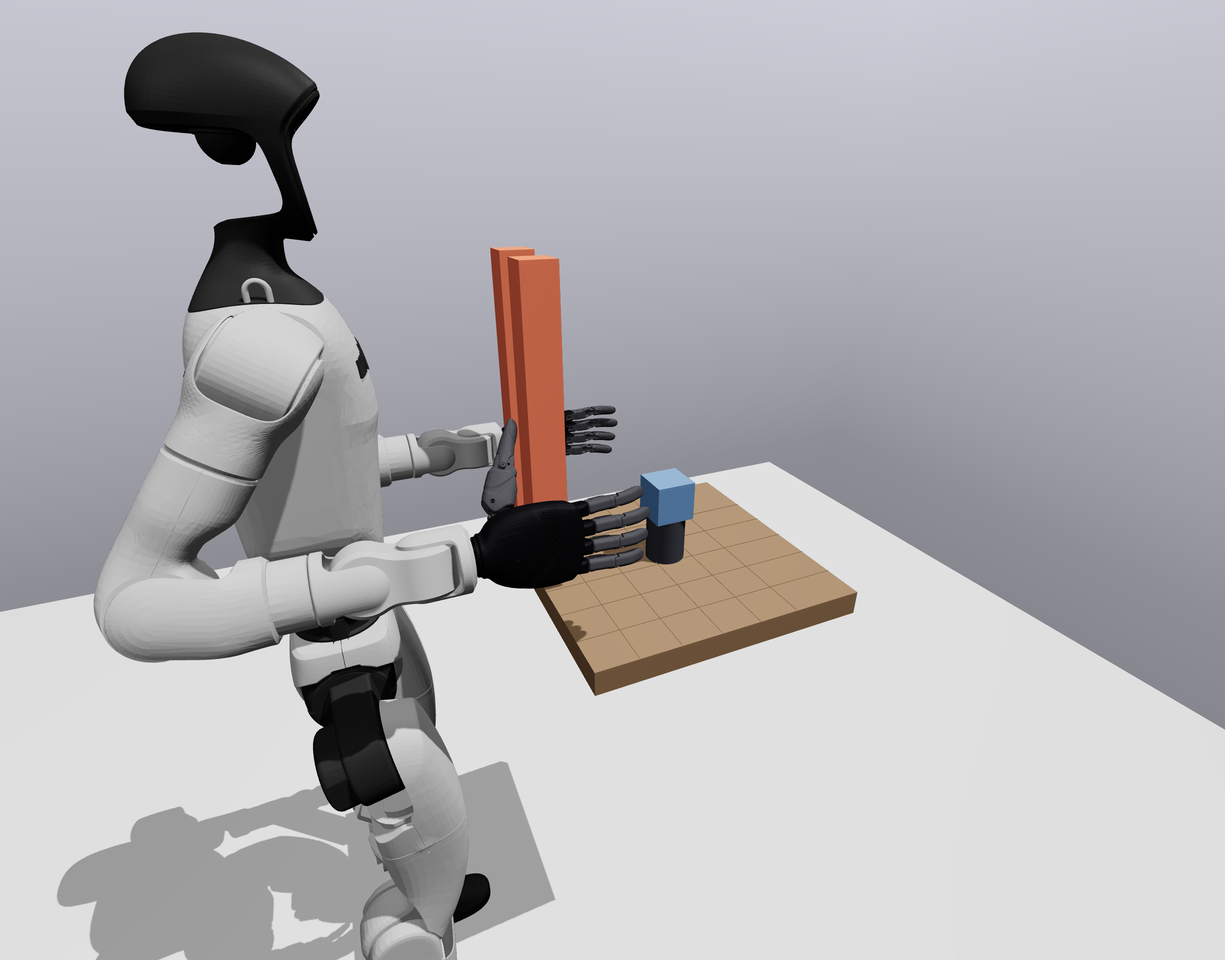} &
\panel[trim=150pt 200pt 294pt 50pt, clip]{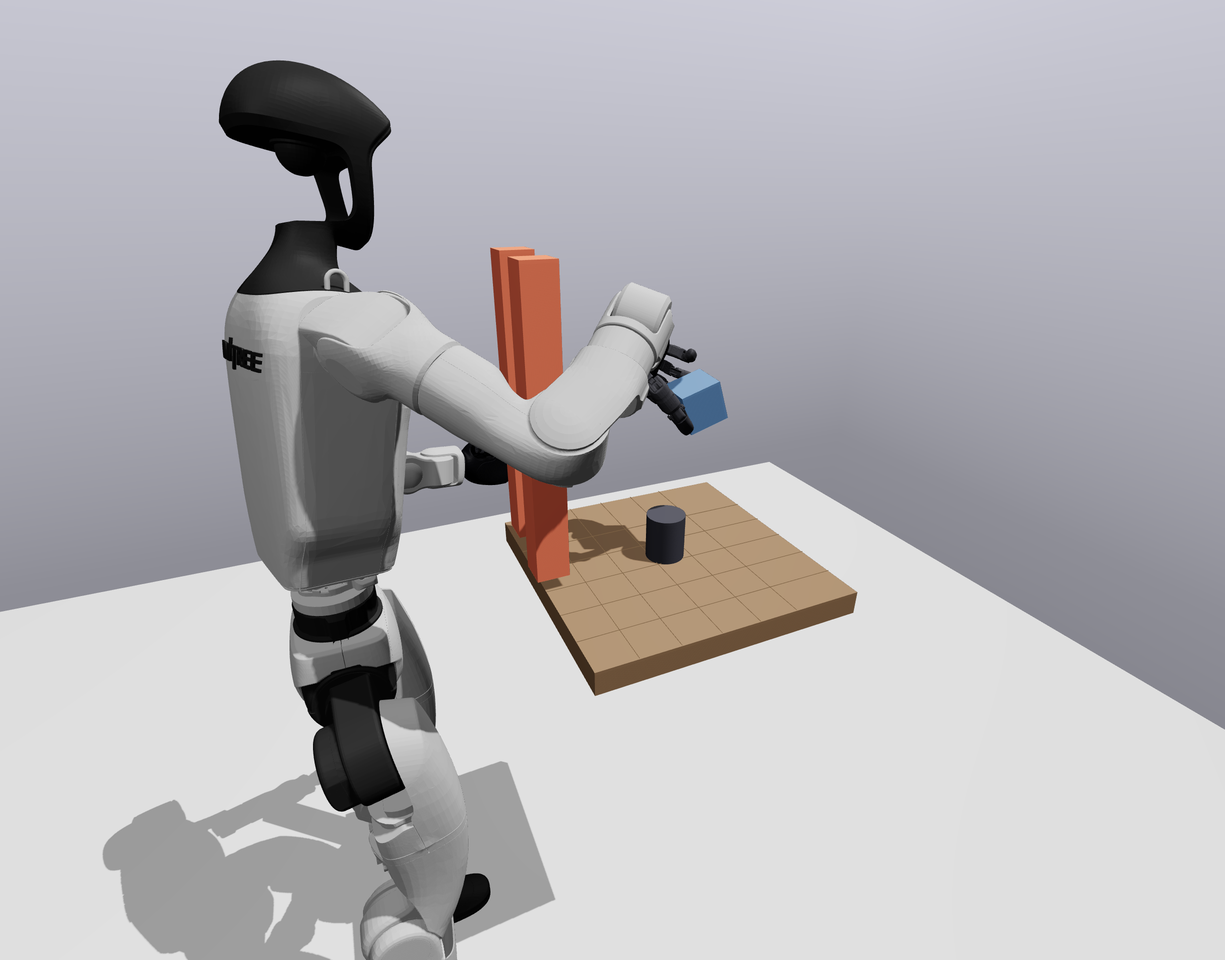} \\[2pt]
\multicolumn{2}{c}{\footnotesize (a) Obstacle Pair} &
\multicolumn{2}{c}{\footnotesize (b) Dynamic Obstacle} &
\multicolumn{2}{c}{\footnotesize (c) Blocked Approach} &
\multicolumn{2}{c}{\footnotesize (d) Unitree G1 Safe Lift Task}
\end{tabular}
\captionof{figure}{\textbf{Plan-Free Grasp Execution Across Four Task Settings.} We compare the initial and final configuration of our field controller in each of four reach-avoid-stay settings (\Cref{sec:experiments}). In (a), a tabletop manipulator grasps an object obstructed by a \(6\times 6 \times 50\,\mathrm{cm}\) obstacle pair. In (b), the same manipulator avoids a dynamic obstacle crossing its approach path. In (c), we consider a denser tabletop workspace comprising an additional obstacle obstructing the approach window, forcing the field around a longer yet safe course. In (d), we consider a different embodiment class altogether, a Unitree G1 humanoid, lifting an object placed behind a \(4\times 4 \times 36\,\mathrm{cm}\) obstacle.}
\label{fig:teaser}
\end{figure*}

By contrast, grasp fields provide the missing target without addressing execution. Neural grasp distance fields learn the distance from a query pose to the manifold of valid grasps for an object, interpret it as a cost, and minimize it inside a trajectory optimizer jointly with smoothness and collision terms \cite{weng2023ngdf}. Such fields are defined in task space, which requires solving inverse kinematics at execution time. Configuration-space distance fields resolve this for obstacle geometry, demonstrating that a distance representation in joint space admits one-step gradient projection where task-space signed distance fields fail on the nonlinearity of the kinematics \cite{li2024cdf}. Our field extends this construction from obstacle boundaries to a discrete set of grasp configurations, and our controller evaluates the field at every control step, whereas the prior work embeds it in an optimizer. Unlike prior distance-field controllers, our framework also certifies a lower bound on the realized force-closure grasp quality throughout execution via a wrench-quality CBF.

Beyond grasp-target fields, candidate grasp sets are widely available. Grasp generators produce grasp poses or full hand configurations for grippers and dexterous hands alike, through analytical \cite{miller_graspit_2004,liFroggerFastRobust2023} and generative~\cite{sundermeyer2021contact,wang_dexgraspnet_2023,enweremEquiDexFlow2026} approaches. The generated grasps satisfy some desired quality metric, from deterministic force closure~\cite{murrayMathematicalIntroductionRobotic2017a,ferrari1992planning} to robust quality measures~\cite{li2023pong,liFroggerFastRobust2023,enwerem2026firmgrasp} that define analytic bounds on the probability of force closure under uncertain friction and object normals. Given a candidate grasp set, our controller operates directly on the individual grasps through a tuple of the object-relative hand-root pose and the corresponding hand configuration (a \emph{grasp descriptor}, \cref{def:targetset}). Standard motion planning frameworks plan an approach motion to a selected pose, with closed-loop grasping typically incorporated as a learned capability within individual systems \cite{morrison2018ggcnn}, and end-to-end policies embed execution in their network weights \cite{chi2023diffusion,black2024pi0}. These methods, however, do not supply a grasp-set target that a reactive controller can query at arbitrary configurations. Our grasp distance fields do.

\subsection{Contributions}
\label{sec:contributions}
Specifically, we make five contributions.
\begin{enumerate}[label=\roman*.]
  \item \emph{Grasp Distance Field and Smooth Approximation}: we introduce the notion of a \emph{grasp distance field} defined over a discrete grasp candidate set in configuration space and establish its two-sided distance bound, smoothness, and bounded-gradient properties. The softmax weights of the resulting field encode the candidate selection as a continuous combination, which eliminates the discrete selection step (\cref{sec:field}).
  \item \emph{A Certified Safe Execution Law}: we give a nominal control along the field's negative gradient, filtered through a combined control Lyapunov and control barrier function quadratic program whose barrier constraints cover self-collision, workspace, object, and obstacle clearance, whose Lyapunov constraints record impeded progress toward the grasp as explicit slack, and whose equality block encodes the dependent-joint constraints of an underactuated hand. We prove that the resulting closed loop renders the safe set forward invariant, and our wrench-quality constraint holds the risk-adjusted force-closure certificate of \cite{enwerem2026firmgrasp} within a prescribed tolerance from hold onset (\cref{sec:filter}).
  \item \emph{A Contact-Mode Switch with Hysteresis}: we give the mode structure that takes execution through contact, the single event a smooth field cannot express, stated as a hybrid system with per-mode constraint sets (\cref{sec:modes}).
  \item \emph{A Wrench-Certificate Bridge}: we restrict admission to grasps whose risk-adjusted force-closure margin \cite{enwerem2026firmgrasp} exceeds a threshold and evaluate the executed grasp under the same margin, connecting the joint-space field to the wrench-space certificates of the grasping literature and reporting the min-weight margin as a baseline (\cref{sec:bridge}).
  \item \emph{Full-Model Verification}: we verify the controller numerically on a 7-degree-of-freedom arm with an 11-joint multifingered hand, including a 50-object grasp descriptor evaluation, a moving-obstacle case on the same constraints, and a humanoid (Unitree G1) instantiation that runs the controller without modification. In a recorded execution over a five-candidate set, the per-step softmax weights concentrate on the nearest candidate at every step, with no separate selection step in the loop (\cref{sec:experiments}).
\end{enumerate}

\subsection{Organization}
\label{sec:organization}
We review the wrench-space certificates and safety-filter constructions our controller builds on in \cref{sec:prelim} and state the problem in \cref{sec:problem}. In \cref{sec:field}, we define the grasp distance field and establish its properties. In \cref{sec:filter}, we give the safe execution law in three parts, the filtered program and its invariance guarantee (\cref{sec:filter-qp}), the contact-mode switch (\cref{sec:modes}), and the bridge from the field to the wrench certificates (\cref{sec:bridge}). We report the numerical experiments in \cref{sec:experiments} and discuss limitations and extensions in \cref{sec:discussion}.

\section{Preliminaries}
\label{sec:prelim}
\paragraph*{Notation}
\(\Real\) and \(\Realnn\) denote the real and nonnegative real numbers. For a symmetric positive definite \(\Lambda\), \(\lVert x \rVert_{\Lambda} = (x^{\top}\Lambda x)^{1/2}\), and \(\lVert x \rVert\) with no subscript is the Euclidean norm. \(\Int(A)\), \(\partial A\), and \(\conv(A)\) denote interior, boundary, and convex hull. A continuous function \(\alpha \colon \Real \to \Real\) is of extended class \(\mathcal{K}_\infty\) if it is strictly increasing, \(\alpha(0)=0\), and \(\alpha(r) \to \pm\infty\) as \(r \to \pm\infty\), and a continuous function \(\gamma \colon \Realnn \to \Realnn\) is of class \(\mathcal{K}\) if it is strictly increasing with \(\gamma(0) = 0\). Throughout, \(\alpha\) denotes a barrier rate and \(\gamma\) a convergence rate. Rigid transforms are elements of \(\SEthree\), and \(X_{AB} \in \SEthree\) denotes the pose of frame \(B\) expressed in frame \(A\). \(\mathbf{1}\) is the vector of ones, \(\mathbb{I}[\cdot]\) is the scalar indicator (with \([\cdot]\) the Iverson bracket~\cite{iverson1962programming,knuth1992notation}), equal to \(1\) when the condition (predicate) in brackets holds and \(0\) otherwise, and \(a \succeq b\) is the component-wise inequality.

\subsection{Wrench Geometry and Grasp Certificates}
\label{sec:prelim-wrench}
We consider a grasp of \(n_c\) fingertip contacts with a target object, under the standard quasi-static contact model \cite{murrayMathematicalIntroductionRobotic2017a,rimonMechanicsRobotGrasping2019}. Contact \(i\) at point \(p_i\) with inward unit normal \(\hat n_i\) applies forces inside a Coulomb friction cone with coefficient \(\mu_i\), and the grasp map \(G(q) \in \Real^{6 \times 3 n_c}\) collects the per-contact basis directions so that stacked contact forces \(F_C \in \Real^{3 n_c}\) produce the net object wrench \(w_O = G(q) F_C\). For computation, we linearize each cone into a pyramid with \(n_s\) edges. Every edge force maps through \(G(q)\) to a basis wrench, and the \(m = n_c n_s\) basis wrenches form the columns of the wrench matrix \(W(q) \in \Real^{6 \times m}\), whose convex hull is the grasp wrench space (GWS).

\begin{definition}[Force Closure~\cite{murrayMathematicalIntroductionRobotic2017a}]
\label{def:fc}
A grasp achieves force closure if its wrench matrix satisfies \(0 \in \Int\big(\conv(W(q))\big)\).
\end{definition}

\begin{definition}[Ferrari-Canny Margin~\cite{ferrari1992planning}]
\label{def:eps}
The Ferrari-Canny margin of a grasp is
\(\varepsilon(q) = \max\{\, r \ge 0 \mid B(0,r) \subseteq \conv(W(q)) \,\}\),
the radius of the largest origin-centered ball inscribed in the GWS.
\end{definition}

\begin{definition}[Min-Weight Metric~\cite{liFroggerFastRobust2023}]
\label{def:minweight}
The min-weight margin is \(\lbar(q) := m\,\ell^*(q) \in (-\infty, 1]\), where \(\ell^*(q)\) solves
\begin{equation}
\ell^*(q) = \max_{a \in \Real^{m},\, \ell \in \Real} \; \ell
\ \ \text{s.t.}\ \ W(q)\, a = 0,\; \mathbf{1}^{\top} a = 1,\; a \succeq \ell\,\mathbf{1},
\label{eq:minweight}
\end{equation}
with \(a\) weighting the basis wrenches (written \(\alpha\) in \cite{liFroggerFastRobust2023}, changed here so the barrier rate of \cref{sec:prelim-cbf} retains its symbol). Whenever the columns of \(W(q)\) hold seven affinely independent basis wrenches, \eqref{eq:minweight} is feasible and \(\ell^*\) is positive exactly on force closure \cite[Thm.~1]{liFroggerFastRobust2023}.
\end{definition}
The min-weight metric is almost-everywhere differentiable, and optimizing it through its linear program yields sub-second grasp synthesis~\cite{li_drop_2025,suh_dexterous_2026}. However, uncertain object normals~\cite{li2023pong}, object pose~\cite{weiszPoseErrorRobust2012}, and friction~\cite{enweremVariationalNeuralBeliefParameterizations2026a} degrade grasp-execution success. We therefore generate our candidate sets under the risk-adjusted margin \(\epsb\) of \firmgrasp{}~\cite{enwerem2026firmgrasp}, and report \(\lbar\) as a baseline.

\begin{definition}[Risk-Adjusted Margin~\cite{enwerem2026firmgrasp}]
\label{def:riskmargin}
Let the contact friction coefficient be random with \(\mu \sim p_\mu\) and fix a confidence \(\beta \in (0,1)\). The risk-adjusted friction is the Conditional Value-at-Risk (CVaR) of the friction distribution, \(v_\beta := \mathrm{CVaR}_\beta(\mu)\), the mean of its lowest \(1-\beta\) fraction. The risk-adjusted margin is then the Ferrari-Canny radius of the GWS constructed at the risk-adjusted friction,
\begin{equation}
\epsb(q) := \varepsilon\big(q, v_\beta\big),
\label{eq:riskmargin}
\end{equation}
signed by the origin-to-facet distance when the origin lies outside the hull, where \(\epsb < 0\) measures the margin by which the risk-adjusted friction violates closure. A positive \(\epsb\) certifies force closure with probability at least \(\beta\) under \(p_\mu\), and \(\epsb\) is differentiable in the grasp parameters per \cref{lem:epsgrad}. We write \(\epsb_{\mathrm{desc}}\) for the margin of a stored grasp descriptor (\cref{def:targetset}), evaluated at its synthesized contact set. We write \(\epsb_{\mathrm{exec}}\) for the same margin on the contact set the controller realizes at hold onset. We compare the two in \cref{sec:bridge}.
\end{definition}

In addition, \cref{lem:ordering} bounds the Ferrari-Canny margin from below by the min-weight margin, up to an object-dependent constant. Every stored descriptor satisfies \(\lbar > 0\), and at synthesis we rank candidates by \(\epsb\) and reject every candidate with \(\epsb \le k_{\mathrm{adm}} = 0\).

\begin{lemma}[Certificate Comparison, {\cite[Thm.~2]{liIntrinsicRobustness2024}}]
\label{lem:ordering}
For every object \(\mathcal{O}\) there exists a constant \(K(\mathcal{O}) > 0\) such that \(K(\mathcal{O})\, \ell^*(q) \le \varepsilon(q)\) uniformly over the admissible wrench sets of \(\mathcal{O}\).
\end{lemma}

\begin{lemma}[Intrinsic Robustness, {\cite[Cor.~1]{liIntrinsicRobustness2024}}]
\label{lem:intrinsic}
Let a grasp have margin \(\varepsilon(q) > 0\) with basis wrenches \(\bar w_s\). Every realized grasp whose basis wrenches satisfy \(\lVert w_s - \bar w_s \rVert \le \varepsilon(q)\) for all \(s\) remains force closed.
\end{lemma}

Given these certificates, our wrench-quality constraint requires the gradient of \(\epsb\) at every control step, and \cref{lem:epsgrad} gives it in closed form as the nearest-facet normal mapped through the wrench Jacobians. In practice, each basis wrench factors through the same contact Jacobians the self-collision constraints already differentiate.

\begin{assumption}[Facet Regularity~\cite{qiu_new_2022}]
\label{ass:lp}
At every \(q\) considered, the GWS constructed at the risk-adjusted friction \(v_\beta\) has a unique facet nearest the origin, with six affinely independent basis wrenches as its vertices.
\end{assumption}

\begin{lemma}[Risk-Adjusted Margin Gradient]
\label{lem:epsgrad}
Let \cref{ass:lp} hold at \(q\) with \(\epsb(q) > 0\), and let \(F^\star\) denote the facet of the GWS, constructed at the risk-adjusted friction \(v_\beta\), nearest the origin, with vertex index set \(S^\star\), outward unit normal \(u^\star(q) \in \Real^6\), and convex weights \(a^\star_s(q) \ge 0\), \(\sum_{s \in S^\star} a^\star_s(q) = 1\), that locate the projection of the origin onto \(F^\star\). Then \(\epsb\) is differentiable at \(q\) and
\begin{equation}
\nabla_q \epsb(q) = \sum_{s \in S^\star} a^\star_s(q)\, u^\star(q)^{\!\top} \frac{\partial w_s(q)}{\partial q}.
\label{eq:epsgrad}
\end{equation}
\end{lemma}

\begin{proof}
With the origin interior to the hull, the inscribed radius of \cref{def:riskmargin} equals the distance from the origin to the nearest facet. Hence, \(\epsb(q) = u^\star(q)^{\!\top} w_s(q)\) for every \(s \in S^\star\), since the supporting hyperplane of \(F^\star\) holds each of its vertices at the common offset. Affine independence of the six vertices makes \((u^\star, \epsb)\) the unique solution of a nonsingular linear system, \(u^{\!\top} w_s = \epsb\) over \(s \in S^\star\) with \(\lVert u \rVert = 1\), and makes \(a^\star\) the unique weight vector of the projection \(p^\star = \epsb\, u^\star = \sum_{s \in S^\star} a^\star_s w_s\). By the implicit function theorem, \(u^\star\) and \(a^\star\) therefore extend in a continuously differentiable manner to a neighborhood of \(q\), and uniqueness of the nearest facet ensures the minimum remains attained on \(F^\star\) there. Differentiating \(\epsb = \sum_{s \in S^\star} a^\star_s\, u^{\star\top} w_s\) gives
\[
d\epsb = \sum_{s \in S^\star} a^\star_s\, u^{\star\top}\, dw_s
       + \sum_{s \in S^\star} \big(u^{\star\top} w_s\big)\, da^\star_s
       + p^{\star\top}\, du^\star,
\]
where the middle term collects to \(\epsb \sum_{s} da^\star_s = 0\) since the weights sum to one, the last term is \(\epsb\, u^{\star\top} du^\star = \tfrac{1}{2}\,\epsb\; d\lVert u^\star \rVert^2 = 0\) on the unit sphere, and the remaining term is \eqref{eq:epsgrad}.
\end{proof}

\subsection{Safety Filters}
\label{sec:prelim-cbf}

Consider the kinematic model \(\dot q = v\) on \(\Cfg \subseteq \Real^{n}\) with input \(v \in \Real^{n}\), and a continuously differentiable \(h_j \colon \Cfg \to \Real\) defining the set
\begin{equation}
\Sset_j = \{\, q \in \Cfg \mid h_j(q) \ge 0 \,\}, \qquad \Sset = \textstyle\bigcap_{j=1}^{N_b} \Sset_j.
\label{eq:safeset}
\end{equation}

\begin{definition}[Control Barrier Function, {\cite{wieland2007constructive,ames2017cbfqp}}]
\label{def:cbf}
The function \(h_j\) is a control barrier function for \(\dot q = v\) on \(\Sset_j\) if there exists an extended class-\(\mathcal{K}_\infty\) function \(\alpha\) such that for every \(q\),
\(\sup_{v} \nabla h_j(q)^{\top} v \ge -\alpha(h_j(q))\).
\end{definition}

Given \cref{def:cbf}, every locally Lipschitz controller satisfying the inequality of \cref{def:cbf} on every constraint renders \(\Sset\) forward invariant, and the minimally invasive such controller solves a quadratic program for the command nearest the nominal command, subject to the barrier constraints \cite{ames2017cbfqp,ames2019cbf}. We adopt \(\alpha(s) = \alpha_0\, s\), the standard choice for velocity-level manipulator control \cite{singletary2022kinematic}. Filtering at the velocity level extends to the full-order dynamics whenever a low-level loop tracks the commanded velocity exponentially, with a tracking-error term tightening the certified set \cite[Thm.~2]{singletary2022kinematic}.

\subsection{Convergence Filters}
\label{sec:prelim-clf}

The barrier constraints of \cref{sec:prelim-cbf} restrict the state to \(\Sset\), and our controller must also reach the grasp target. A control Lyapunov function certifies the requirement to reach the grasp target, through a decrease condition of the same pointwise form as \cref{def:cbf}.

\begin{definition}[Control Lyapunov Function, {\cite{artstein1983stabilization,sontag1989universal,ames2017cbfqp}}]
\label{def:clf}
Let \(\mathcal T \subseteq \Cfg\) be a target set and \(V \colon \Cfg \to \Realnn\) continuously differentiable off \(\mathcal T\) with \(\mathcal T = \{q \mid V(q) = 0\}\). The function \(V\) is a control Lyapunov function (CLF) for \(\dot q = v\) toward \(\mathcal T\) if there exists \(\gamma \in \mathcal K\) such that for every \(q \notin \mathcal T\),
\(\inf_{v} \nabla V(q)^{\top} v \le -\gamma(V(q))\).
\end{definition}

However, the CLF inequality bounds \(v\) from the side opposite the CBF inequality, and the two can therefore conflict. The standard resolution enforces the CBFs as hard constraints and relaxes the CLF with a slack \(\sigma \ge 0\) penalized at weight \(\eta > 0\), the combined CLF-CBF quadratic program \cite{ames2017cbfqp}. We instantiate the CLF-CBF program twice in \cref{sec:filter}, on our field and on the risk-adjusted margin (\cref{def:riskmargin}).

\section{Problem Formulation}
\label{sec:problem}

We consider an arm-hand system with configuration \(q = (q_a, q_h) \in \Cfg \subseteq \Real^{n_a + n_h}\), where \(n_a\) and \(n_h\) are the numbers of arm and hand degrees of freedom, under the single-integrator kinematic model
\begin{equation}
\dot q = v, \qquad v \in \Real^{n_a + n_h}.
\label{eq:kinematics}
\end{equation}
The scene holds an object model, a static environment, and a possibly time-varying set of obstacle bodies. Suppose further that the hand is underactuated. A set \(\mathcal{M}\) of dependent-joint pairs \((j_{\mathrm{f}}, j_{\mathrm{b}})\) with multipliers \(\nu_m\) then constrains realizable velocities to \(\ker A_m = \{\, v \mid \dot q_{j_{\mathrm{f}}} = \nu_m\, \dot q_{j_{\mathrm{b}}},\ (j_{\mathrm{f}}, j_{\mathrm{b}}) \in \mathcal{M} \,\}\), where \(j_{\mathrm{f}}\) indexes the follower joint of the pair and \(j_{\mathrm{b}}\) the lead joint.

\begin{definition}[Grasp Descriptor and Target Set]
\label{def:targetset}
A grasp descriptor is a pair \((X_{OH}, q_h^{\mathrm{g}}) \in \SEthree \times \Real^{n_h}\), the hand-root pose expressed in the object frame together with a closure hand shape. Given an object pose \(X_{WO} \in \SEthree\), a candidate pair \((\qpre{i}, \qgrasp{i}) \in \Cfg \times \Cfg\) realizes the grasp descriptor if the forward kinematics at \(\qgrasp{i}\) yield the hand-root pose \(X_{WO} X_{OH}\) with hand coordinates \(q_h^{\mathrm{g}}\). The pregrasp \(\qpre{i}\) yields the same pose, backed off a fixed Cartesian offset along the descriptor's approach axis with the fingers open. The grasp target set is \(\Gset = \{\qpre{i}\}_{i=1}^{N}\) over the realized candidates.
\end{definition}

\begin{problem}
\label{prob:main}
Given the system \eqref{eq:kinematics}, a candidate set per \cref{def:targetset}, and barrier constraints \(h_1, \dots, h_{N_b}\) covering self-collision, workspace, object, and obstacle clearance, design a stationary feedback law \(v = \kappa(q)\), evaluated pointwise with no prediction horizon, such that (i) the closed loop renders \(\Sset\) forward invariant along solutions whose velocities lie in \(\ker A_m\), (ii) away from contact the closed-loop trajectory decreases a distance to \(\Gset\), and (iii) a mode-switching structure admits fingertip contact, holds the grasp, and lifts the object, while (i) remains with the constraint set the mode switch prescribes. Conditions (i) through (iii) compose a reach-avoid-stay specification, the reach-avoid problem of \cite{fisac2015reachavoid} extended with an invariance requirement in the sense of \cite{meng2023lyapunov}.
\end{problem}

However, no smooth construction satisfies item (iii), since contact admission violates the object clearance constraints by design. We therefore resolve it with a hysteresis-based mode structure in \cref{sec:modes}.

\section{Grasp Distance Fields in Configuration Space}
\label{sec:field}
\begin{definition}[Grasp Distance Field]
\label{def:hardfield}
For a candidate set \(\Gset\) with cardinality \(N > 0\) and a diagonal metric \(\Lambda \succ 0\), the grasp distance field is the set distance
\begin{equation}
d_{\min}(q) = \min_{i} \, d_i(q),
\qquad
d_i(q) = \big\lVert \qpre{i} - q \big\rVert_{\Lambda},
\label{eq:hardfield}
\end{equation}
where \(\Lambda\) weights arm coordinates against hand coordinates.
\end{definition}

However, the field of \cref{def:hardfield}, whose zero-level set is \(\Gset\), is not a suitable control objective. The minimum is non-smooth wherever two candidates lie at the same distance, the Voronoi boundaries of \(\Gset\), and the gradient jumps across them. A law that follows this gradient therefore produces chatter precisely where the implicit selection between candidates must switch. Our controller therefore follows the negative gradient of a smooth approximation, and \cref{fig:concept} illustrates both fields and the negative-gradient trajectories on a planar example.

\begin{definition}[Softmin Grasp Distance Field]
\label{def:field}
For the field of \cref{def:hardfield} and a smoothing parameter \(\rho > 0\), the softmin grasp distance field is
\begin{equation}
d_G(q) = -\tfrac{1}{\rho}\,\log \sum_{i=1}^{N} \exp\!\big({-\rho\, d_i(q)}\big).
\label{eq:field}
\end{equation}
\end{definition}

\begin{lemma}[Unit Gradients]
\label{lem:unitgrad}
For \(q \ne \qpre{i}\), \(\nabla d_i(q) = -\Lambda(\qpre{i} - q)/d_i(q)\) and \(\lVert \nabla d_i(q) \rVert_{\Lambda^{-1}} = 1\).
\end{lemma}

\begin{proof}
Differentiate \(d_i^2 = (\qpre{i}-q)^{\top}\Lambda(\qpre{i}-q)\) and evaluate \(\nabla d_i^{\top} \Lambda^{-1} \nabla d_i = (\qpre{i}-q)^{\top}\Lambda(\qpre{i}-q)/d_i^2 = 1\).
\end{proof}

\begin{proposition}[Field Properties]
\label{prop:field}
On the set where every \(d_i(q) > 0\), our softmin field (\cref{def:field}) satisfies
\begin{equation}
d_{\min}(q) - \tfrac{\log N}{\rho} \;\le\; d_G(q) \;\le\; d_{\min}(q),
\label{eq:sandwich}
\end{equation}
it is smooth there, and its gradient is the convex combination
\begin{equation}
\nabla d_G(q) = \sum_{i=1}^{N} \beta_i(q)\, \nabla d_i(q),
\quad
\beta_i(q) = \frac{e^{-\rho d_i(q)}}{\sum_j e^{-\rho d_j(q)}},
\label{eq:grad}
\end{equation}
with \(\lVert \nabla d_G(q) \rVert_{\Lambda^{-1}} \le 1\).
\end{proposition}

\begin{proof}
The sum in \eqref{eq:field} lies between \(e^{-\rho d_{\min}}\) and \(N e^{-\rho d_{\min}}\), which gives \eqref{eq:sandwich} after taking logarithms. Away from the candidate points, each \(d_i\) is smooth, the log-sum-exp composition preserves smoothness, and \eqref{eq:grad} follows by the chain rule. By \cref{lem:unitgrad}, each term has \(\Lambda^{-1}\)-norm one, and a convex combination of vectors of \(\Lambda^{-1}\)-norm one has \(\Lambda^{-1}\)-norm at most one.
\end{proof}
\begin{figure*}[t]
\centering
\includegraphics[width=.32\textwidth]{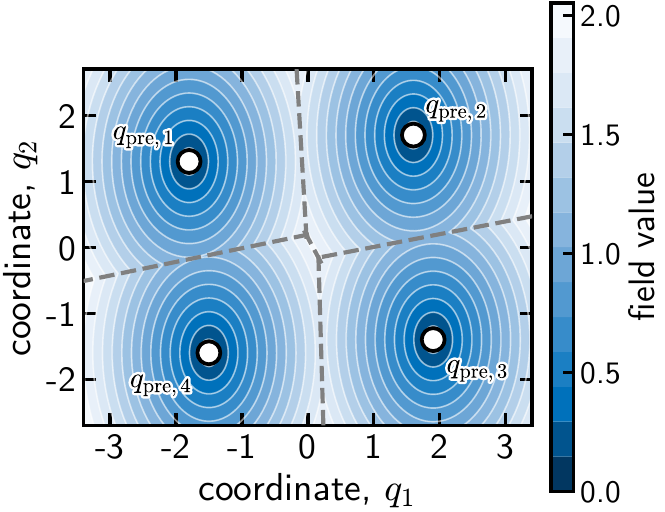}\hfill
\includegraphics[width=.32\textwidth]{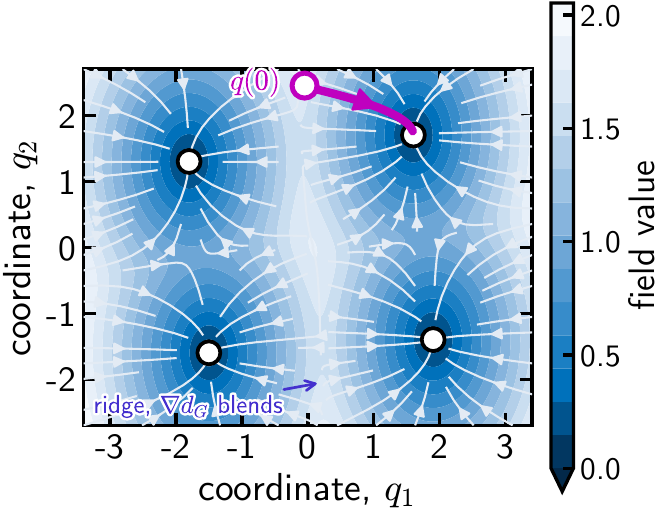}\hfill
\includegraphics[width=.32\textwidth]{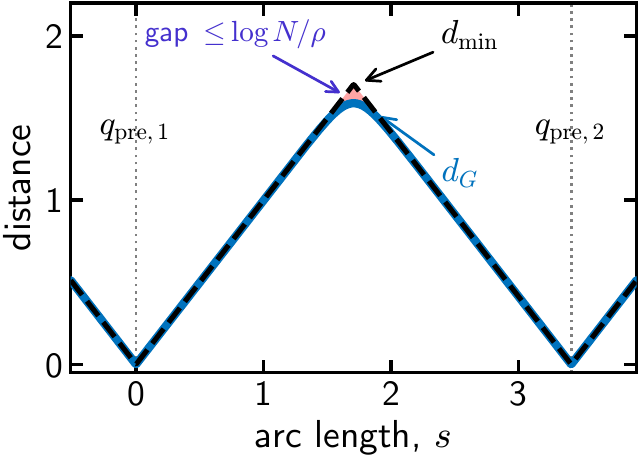}
\caption{\textbf{The Grasp Distance Field on a Planar Toy System.} We evaluate the construction on a planar two-candidate system. Our hard field (\cref{def:hardfield}) is non-smooth along the Voronoi ridges of the candidate set (left). Our softmin field (\cref{def:field}) is smooth across these ridges and preserves each minimum, and its negative-gradient trajectories converge to whichever candidate lies nearest, the implicit selection of \eqref{eq:grad} (middle). A section through two candidates shows the softmin tracking the hard minimum within the \(\log N/\rho\) gap of \eqref{eq:sandwich} (right).}
\label{fig:concept}
\end{figure*}
Given \cref{prop:field}, the nominal command follows the field's negative gradient,
\begin{equation}
v_{\mathrm{nom}}(q) = -k\,\nabla d_G(q), \qquad k > 0,
\label{eq:nominal}
\end{equation}
so \(\lVert v_{\mathrm{nom}} \rVert_{\Lambda^{-1}} \le k\) everywhere by \cref{prop:field}. No selection step comes before \eqref{eq:nominal}. The softmax weights \(\beta_i\) of \eqref{eq:grad} are the selection, revised at every evaluation, and the mode switch of \cref{sec:modes} reads the dominant index only for a closure target.

\begin{proposition}[Unconstrained Decrease]
\label{prop:descent}
Along solutions of \(\dot q = v_{\mathrm{nom}}(q)\) on the set where \(d_G\) is smooth,
\(\tfrac{d}{dt}\, d_G(q(t)) = -k\, \lVert \nabla d_G(q(t)) \rVert^2 \le 0\),
with equality exactly at stationary points of the field.
\end{proposition}

\begin{proof}
Chain rule with \eqref{eq:nominal}.
\end{proof}

\begin{remark}[The Field Is a CLF]
\label{rem:fieldclf}
By \cref{prop:descent}, \(d_G\) is a control Lyapunov function per \cref{def:clf} toward \(\Gset\) on every compact set where \(\lVert \nabla d_G \rVert_{\Lambda^{-1}} \ge g_0 > 0\), with a linear rate determined by \(g_0\) and the size of the set. The property fails at the stationary points of the softmin between competing candidates, which we report as trapped equilibria of the filtered closed loop in \cref{sec:discussion}. When the barrier constraints of \cref{sec:filter} act against the negative gradient, the filtered closed loop can stop short of \(\Gset\) with no record of the impeded progress. In \cref{sec:filter}, we therefore restate the same inequality as an explicit slack-bearing constraint.
\end{remark}

\section{Safe Execution with Guaranteed Quality}
\label{sec:filter}

\subsection{The Filtered Program}
\label{sec:filter-qp}

The filtered command solves a single quadratic program per control step,
\begin{equation}
\begin{aligned}
v^\star(q) = \argmin_{v,\,\sigma \succeq 0} & \tfrac12 \lVert v - v_{\mathrm{nom}}(q) \rVert^2 + \eta \textstyle\sum_{r} \sigma_r^{2}\!\!\\
\text{s.t.}\; & \nabla h_j(q)^{\!\top} v \ge -\alpha_0 h_j(q), \, j = 1, \dots, N_b,\\
& \nabla h_{\mathrm{wq}}(q)^{\!\top} v \ge -\alpha_0 h_{\mathrm{wq}}(q) \text{ in }\textsc{h},\textsc{l},\\
& \nabla V_r(q)^{\!\top} v \le -\gamma_r(V_r(q)) + \sigma_r, \, r \in \mathcal R(p),\\
& A_m v = 0, \;\; \lvert v \rvert \le \bar v,
\end{aligned}
\label{eq:qp}
\end{equation}
where \(\bar v\) collects the per-joint velocity limits, \(A_m\) encodes the dependent-joint constraints of \cref{sec:problem}, \(\sigma\) collects the slack variables \(\sigma_r\), and \(\mathcal R(p)\) is the set of control Lyapunov constraints supplied by the active mode \(p \in \mathcal P\) of \cref{sec:modes}, as in
\begin{equation}\label{eq:clfexpa}
V_r(q) \;=\;
\begin{cases}
d_G(q), & p = \textsc{r},\\[1pt]
\tfrac12 \big\lVert q_f - \qgrasp{f,i^\star} \big\rVert_{\Lambda_f}^2, & p \in \{\textsc{c}, \textsc{h}\},\; r = f.
\end{cases}
\end{equation}
In \textsc{reach}, the set \(\mathcal R(\textsc{r})\) holds the field constraint alone, while in \textsc{close} and \textsc{hold} the index \(f\) ranges over the five fingers, one constraint per finger.
Each CLF in \eqref{eq:clfexpa} has a class-\(\mathcal K\) rate \(\gamma_r\) and its own slack, \(\sigma_r\). In \(\mathcal R(\textsc{h})\), we exclude each finger whose \(c_f\) = 1 (a finger in contact, per \eqref{eq:hold}), and we restrict the per-finger diagonal block \(\Lambda_f\) to the corresponding finger's coordinates. Our wrench-quality constraint is the CBF \(h_{\mathrm{wq}}(q) := \epsb(q) - \big(\epsb(q_0) - k_{\mathrm{wq}}\big)\), built on \cref{def:riskmargin} and enforced from the moment our mode switch admits a grasp. Here \(q_0 := q(t_0)\) is the configuration at the moment of admission, and \(k_{\mathrm{wq}}\) bounds the decay we allow from it. We state the bound relative to \(\epsb(q_0)\). A bound fixed at the synthesis threshold \(k_{\mathrm{adm}}\) would be active on those grasps from the first control step, since at the confidence \(\beta\) of \cref{sec:setup}, the risk-adjusted margin is negative over much of the descriptor set at the risk-adjusted friction.

Specifically, the collection \(\{h_j\}\) comprises four CBF families. Self-collision constraints bound the witness-point distance of every geometry pair the collision model leaves active, with gradients through the translational point Jacobians of the two nearest points. Workspace constraints bound every hand-assembly collision sphere inside the half-spaces of the static scene. Object constraints hold the same spheres clear of the target object during approach, with the active subset prescribed by the mode switch of \cref{sec:modes}. Obstacle constraints, their barriers denoted \(h_{\mathrm{obs}}\), encode box restricted regions for the hand assembly and the arm links in every mode, with obstacle poses updated at every step. Given these families, the invariance argument of \cref{thm:safe} relies on regularity along the closed loop, which \cref{ass:qp} collects.

\begin{assumption}[QP Regularity]
\label{ass:qp}
Along the closed-loop solution, each \(h_j\) is continuously differentiable in a neighborhood of the solution, each active \(V_r\) is continuously differentiable off its target set, \cref{ass:lp} holds whenever \(h_{\mathrm{wq}}\) is active, \eqref{eq:qp} is feasible, and \(v^\star(\cdot)\) is locally Lipschitz.
\end{assumption}

For a CLF-CBF program whose vector fields, barrier and Lyapunov gradients, and cost terms are locally Lipschitz, and whose barrier has relative degree one, \cite[Thm.~3]{ames2017cbfqp} establishes that the solution is locally Lipschitz. Our program \eqref{eq:qp} adds the input bounds \(\lvert v \rvert \le \bar v\) and the equality block \(A_m v = 0\), a case for which \cite{ames2017cbfqp} reports no such guarantee, and \cref{ass:qp} therefore posits the regularity in place of deriving it. We report feasibility empirically at every step of every trial in \cref{sec:experiments}.

\begin{theorem}[Safe Execution]
\label{thm:safe}
Under \cref{ass:qp}, the closed loop \(\dot q = v^\star(q)\) satisfies, for every hard constraint \(h_j\) or \(h_{\mathrm{wq}}\) and every \(t \ge t_0\) in the maximal interval of existence, writing \(t_0\) for the time the constraint enters the active set,
\begin{equation}
h(q(t)) \;\ge\; h(q(t_0))\, e^{-\alpha_0 (t-t_0)}.
\label{eq:hbound}
\end{equation}
In particular, if \(h(q(t_0)) \ge 0\) then the constraint stays nonnegative for all such \(t\), and a constraint with \(h(q(t_0)) < 0\) approaches its boundary at least exponentially from below.
\end{theorem}

\begin{proof}
Every feasible point of \eqref{eq:qp} satisfies \(\tfrac{d}{dt} h(q(t)) \ge -\alpha_0 h(q(t))\) along the closed loop for every hard constraint, the soft constraints of \eqref{eq:qp} never appear in this inequality, and the comparison lemma yields \eqref{eq:hbound}, the standard CBF-QP invariance argument \cite{ames2017cbfqp} applied beginning from whichever time step the constraint becomes active.
\end{proof}

In particular, for the barrier constraints present from \(t=0\), \(t_0 = 0\) recovers the invariance statement of \(\Sset\) in \eqref{eq:safeset} directly, if \(q(0) \in \Sset\) then \(q(t) \in \Sset\) for all \(t\) in the maximal interval. The mode switch of \cref{sec:modes} changes the active constraint set only at isolated events, the hysteresis band and the two duration limits \(\tau_{\textsc{h}}^{+}, \tau_{\textsc{h}}^{-}\) of \cref{rem:hysteresis} exclude accumulation. The argument therefore applies on each interval between events with \(t_0\) the most recent activation, and \eqref{eq:hbound} extends across them for constraints that remain active.

\begin{remark}[Nonsmooth Constraints]
\label{rem:nonsmooth}
Witness points on convex hulls jump at support switches, the nearest facet of \cref{lem:epsgrad} jumps likewise, and both therefore violate the smoothness of \cref{ass:qp} on a measure-zero set. Input-to-state safety covers this remainder by tightening the certified set under the violation, so invariance holds for the tightened set in place of the original \cite{singletary2022kinematic,kolathaya2019issf}.
\end{remark}

\begin{corollary}[Safety Under Convergence Constraints]
\label{cor:safeclf}
\Cref{thm:safe} holds regardless of whether the soft constraints of \eqref{eq:qp} are present. The comparison-lemma argument in its proof uses only the hard constraints, adding CLF constraints moves the solution further from \(v_{\mathrm{nom}}\) while leaving every barrier constraint at its original bound, and the slack variables \(\sigma_r\) absorb the resulting conflict without entering \eqref{eq:hbound}.
\end{corollary}

\begin{proposition}[Quality Forward Invariance]
\label{prop:qualityinv}
Under \cref{ass:lp}, with \(q_0 := q(t_0)\) the configuration at the time \(h_{\mathrm{wq}}\) enters the active constraint set, \(\epsb(q(t)) \ge \epsb(q_0) - k_{\mathrm{wq}}\) for every \(t \ge t_0\) in the maximal interval of existence. The risk-adjusted margin, in other words, never falls more than \(k_{\mathrm{wq}}\) below its value at hold onset.
\end{proposition}

\begin{proof}
By construction, \(h_{\mathrm{wq}}(q(t_0)) = k_{\mathrm{wq}} \ge 0\), and \cref{thm:safe} applied to \(h_{\mathrm{wq}}\) at \(t_0\) gives \(h_{\mathrm{wq}}(q(t)) \ge h_{\mathrm{wq}}(q(t_0))\, e^{-\alpha_0(t-t_0)} \ge 0\), and \(h_{\mathrm{wq}}(q) \ge 0\) is \(\epsb(q) \ge \epsb(q_0) - k_{\mathrm{wq}}\) by definition.
\end{proof}

Beyond the quality guarantee, the equality block exists because a rigid-body library treats kinematically coupled joints as independent revolute joints. Without \(A_m\), the filtered command would include finger-joint velocities the hardware cannot realize.

\subsection{Contact and Mode Structure}
\label{sec:modes}

Contact admission violates the object constraints by design, and our controller therefore handles first contact through a hybrid structure. We follow the five-tuple hybrid model of \cite{olkin2025chasingstability}, \(\mathcal H = (\mathcal D, \mathcal S, \Gamma, \Delta, \mathcal F)\), collecting the domains, the guards, the allowed transitions, the reset maps, and the per-domain dynamics.

Specifically, the mode index set \(\mathcal{P} = \{\textsc{r}, \textsc{c}, \textsc{h}, \textsc{l}\}\), abbreviating \textsc{reach}, \textsc{close}, \textsc{hold}, and \textsc{lift} wherever a subscript needs a single letter, indexes the domains \(D_p\) of \(\mathcal D\). Every mode shares the same domain, the extended state space \(\mathcal X\) defined below, and the modes differ in the pair \((\Theta_p, v^{\mathrm{nom}}_p)\) of an active object-constraint subset and a nominal command. A per-finger contact indicator \(c \in \{0,1\}^{n_f}\), where \(n_f\) is the number of fingers, encodes the contact state that every guard below depends on. The transition set is \(\Gamma = \{(\textsc{r}, \textsc{c}), (\textsc{c}, \textsc{h}), (\textsc{h}, \textsc{c}), (\textsc{h}, \textsc{l})\}\), carrying the same index as the guards of \eqref{eq:guard}. The indicator \(c\) has no dynamics of its own, \eqref{eq:guard} evaluates it directly from the current tip clearance of each finger, but two of the guards below depend on whether \(c\) has remained on one side of a threshold for a duration, a condition no subset of \(\Cfg\) alone can encode. Two auxiliary states therefore extend the space these two guards range over, the standard construction for guard conditions that depend on elapsed time. The state \(\theta^{+}\) accumulates while \(\mathbf{1}^{\top} c \ge N^{+}\) holds and returns to \(0\) the instant it fails, and \(\theta^{-}\) accumulates likewise while \(\mathbf{1}^{\top} c < N^{-}\) holds. Thus, we define the guards as subsets of the extended state space
\(\mathcal X := \Cfg \times \{0,1\}^{n_f} \times \Real_{\ge 0}^{2}\),
whose points \(x = (q, c, \theta^{+}, \theta^{-})\) contain the configuration,
the contact indicator, and the two duration states, as in
\begin{equation}
\begin{aligned}
\mathcal S_{(\textsc{r},\textsc{c})} &= \{\, x \in \mathcal X \mid d_G(q) \le \delta_{\mathrm{pre}} \,\},\\
\mathcal S_{(\textsc{c},\textsc{h})} &= \{\, x \in \mathcal X \mid \mathbf{1}^{\top} c \ge N^{+} \,\},\\
\mathcal S_{(\textsc{h},\textsc{c})} &= \{\, x \in \mathcal X \mid \mathbf{1}^{\top} c < N^{-},\ \theta^{-} \ge \tau_{\textsc{h}}^{-} \,\},\\
\mathcal S_{(\textsc{h},\textsc{l})} &= \{\, x \in \mathcal X \mid \mathbf{1}^{\top} c \ge N^{+},\ \theta^{+} \ge \tau_{\textsc{h}}^{+} \,\},
\end{aligned}
\label{eq:guard}
\end{equation}
with hysteresis \(N^{-} < N^{+}\) and \textsc{lift} terminal. A guard set alone does not give the evolution of \(\theta^{+}\) and \(\theta^{-}\). Our \(\mathcal F\) holds \(\dot q = v\) together with \(\dot\theta^{+} = \mathbb{I}\big[\mathbf{1}^{\top} c \ge N^{+}\big]\) and \(\dot\theta^{-} = \mathbb{I}\big[\mathbf{1}^{\top} c < N^{-}\big]\), each duration state resetting to \(0\) when its indicator vanishes. The sampled implementation realizes both through the per-step update in \textsc{hold},
\begin{align}
\theta^{+} &\leftarrow \big(\theta^{+} + \Delta t\big)\, \mathbb{I}\big[\mathbf{1}^{\top} c \ge N^{+}\big],
\\
\theta^{-} &\leftarrow \big(\theta^{-} + \Delta t\big)\, \mathbb{I}\big[\mathbf{1}^{\top} c < N^{-}\big],
\label{eq:clocks}
\end{align}
with \(\Delta t\) the control step, both states \(0\) outside \textsc{hold}. The continuous dynamics are otherwise the kinematic model \eqref{eq:kinematics} on every domain, and \eqref{eq:qp} returns the input velocity over the constraint set \(\mathcal R(p)\). Every reset map in \(\Gamma\) is the identity on the configuration \(q\), and both states return to zero on the transitions into and out of \textsc{hold}. \cref{fig:modelogic} graphs several frames of a representative trial alongside the mode band and the guards of \eqref{eq:guard}.

\begin{figure*}[t]
\centering
\setlength{\tabcolsep}{1.5pt}%
\setlength{\fboxsep}{0pt}\setlength{\fboxrule}{0.8pt}%
\setlength{\panelw}{\dimexpr(\textwidth-8\tabcolsep-10\fboxrule)/5\relax}%
\begin{tabular}{@{}ccccc@{}}
\fbox{\includegraphics[width=\panelw]{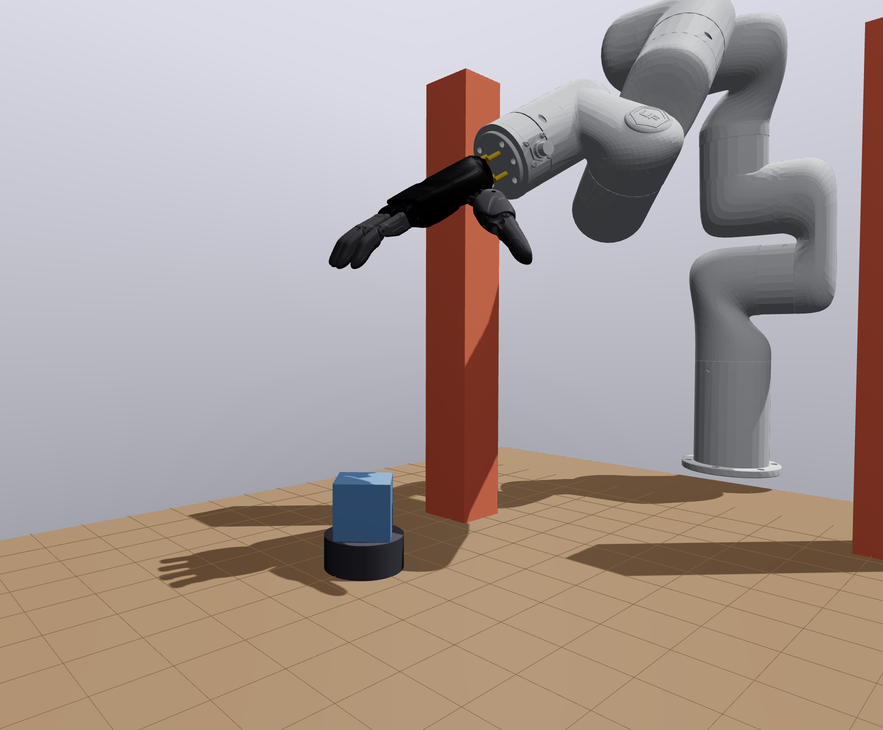}} &
\fbox{\includegraphics[width=\panelw]{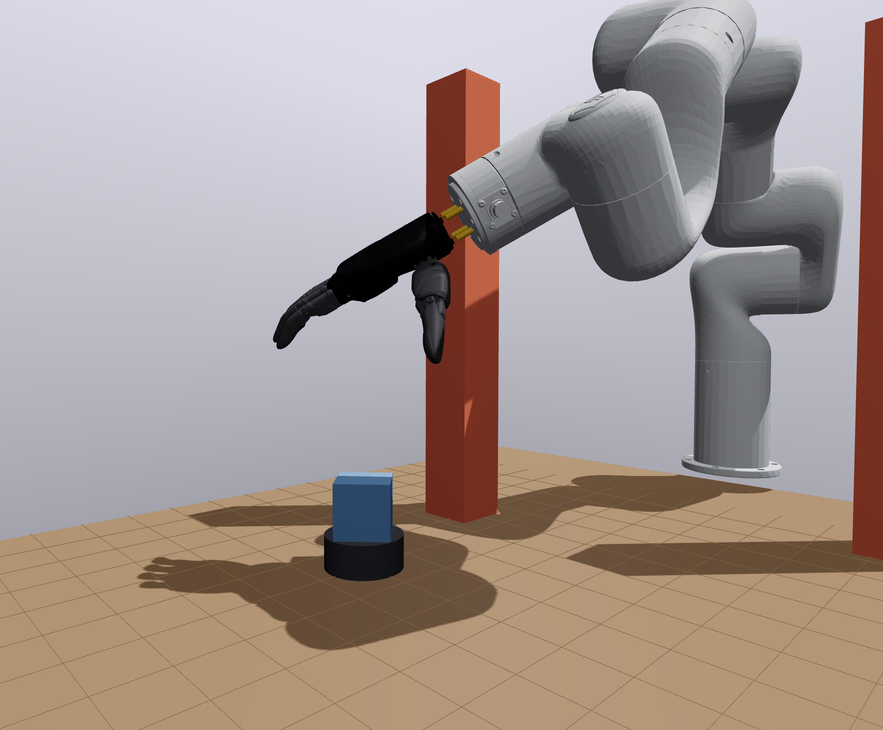}} &
\fbox{\includegraphics[width=\panelw]{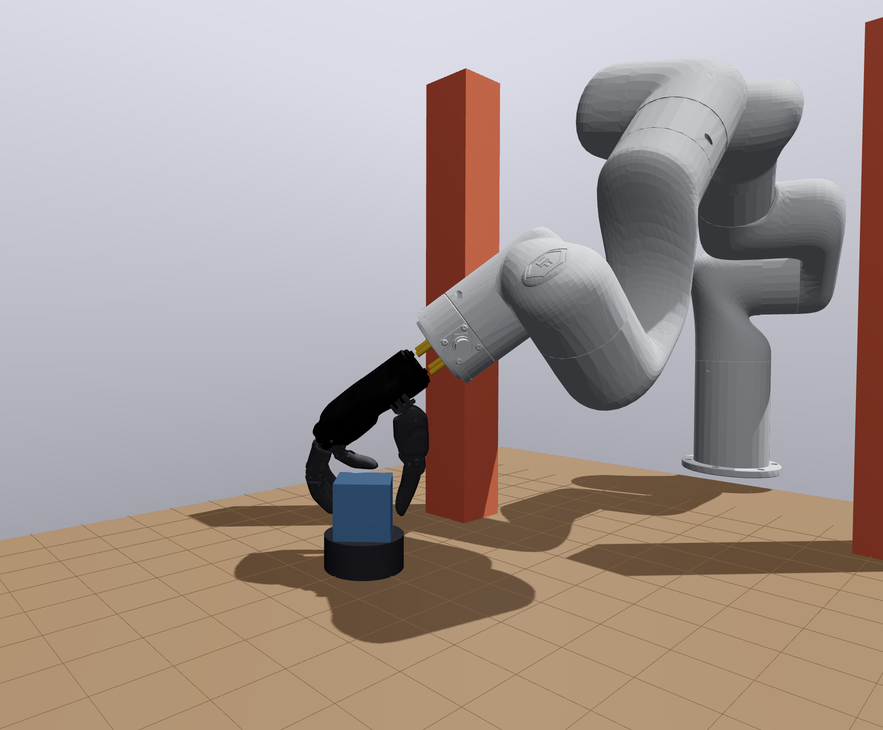}} &
\fbox{\includegraphics[width=\panelw]{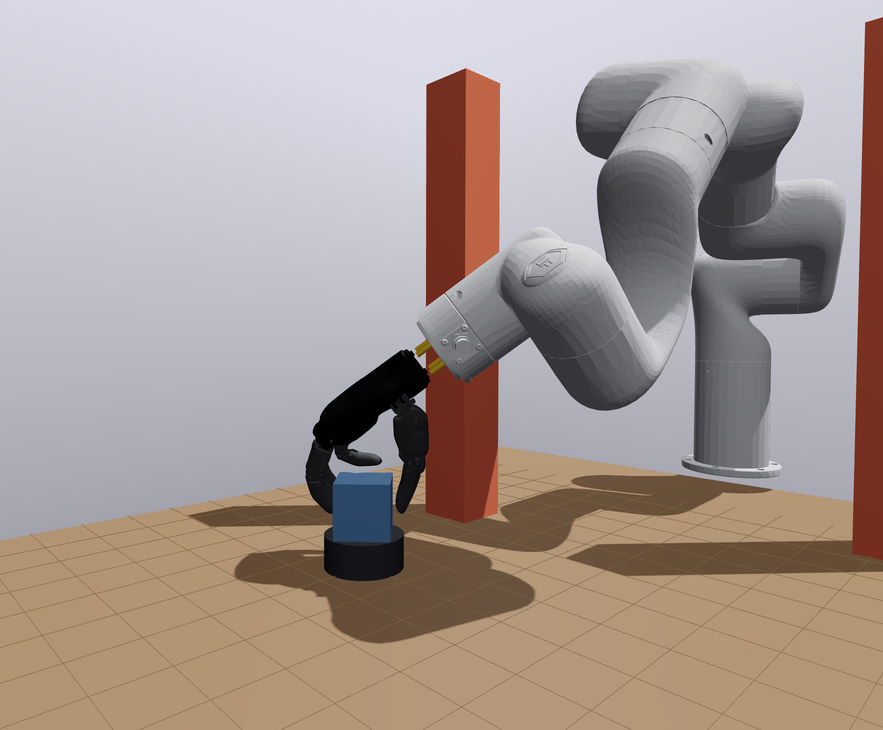}} &
\fbox{\includegraphics[width=\panelw]{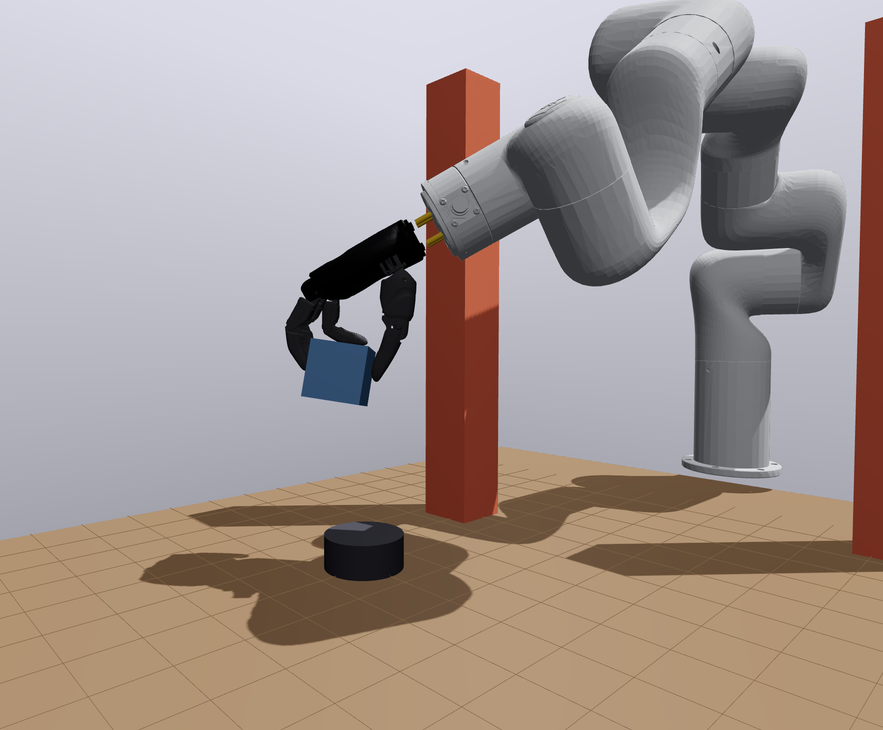}}\\
{\scriptsize \textsc{reach}, step 60} & {\scriptsize \textsc{close} at step 111} & {\scriptsize \textsc{hold} at step 302} & {\scriptsize \textsc{lift} at step 352} & {\scriptsize lift complete, step 461}
\end{tabular}

\vspace{0.5em}
\resizebox{\textwidth}{!}{%
\tikzstyle{mode}=[draw, rounded corners, fill=oursmint, font=\small,
  minimum height=0.7cm, minimum width=2.0cm, align=center]%
\tikzstyle{guard}=[font=\scriptsize, align=center]%
\begin{tikzpicture}[>=stealth]
\node[mode] (r) at (0,0) {\textsc{reach}};
\node[mode] (c) at (4.6,0) {\textsc{close}};
\node[mode] (h) at (9.2,0) {\textsc{hold}};
\node[mode] (l) at (13.8,0) {\textsc{lift}};
\draw[->] (r) -- (c) node[guard, midway, above] {\(d_G \le \delta_{\mathrm{pre}}\)};
\draw[->] (c) -- (h) node[guard, midway, above] {\(\mathbf 1^\top c \ge N^+\)};
\draw[->] (h) -- (l) node[guard, midway, above] {\(\mathbf 1^\top c \ge N^+\)\\\(\theta^+ \ge \tau_{\textsc h}^+\)};
\draw[->] (h) to[bend left=35] node[guard, midway, above]
  {\(\mathbf 1^\top c < N^-\)\\\(\theta^- \ge \tau_{\textsc h}^-\)} (c);
\end{tikzpicture}%
}
\vspace{-2.5em}
\caption{\textbf{Mode Band and Guards for Reach-Avoid-Stay Grasping.} We plot five frames of one trial, one per mode of \cref{sec:modes}, above the transition system with the guards of \eqref{eq:guard} on each edge. Only two edges leave \textsc{hold}, the return to \textsc{close} through the hysteresis band, drawn curved, and the transition to the terminal \textsc{lift} mode.}
\label{fig:modelogic}
\end{figure*}
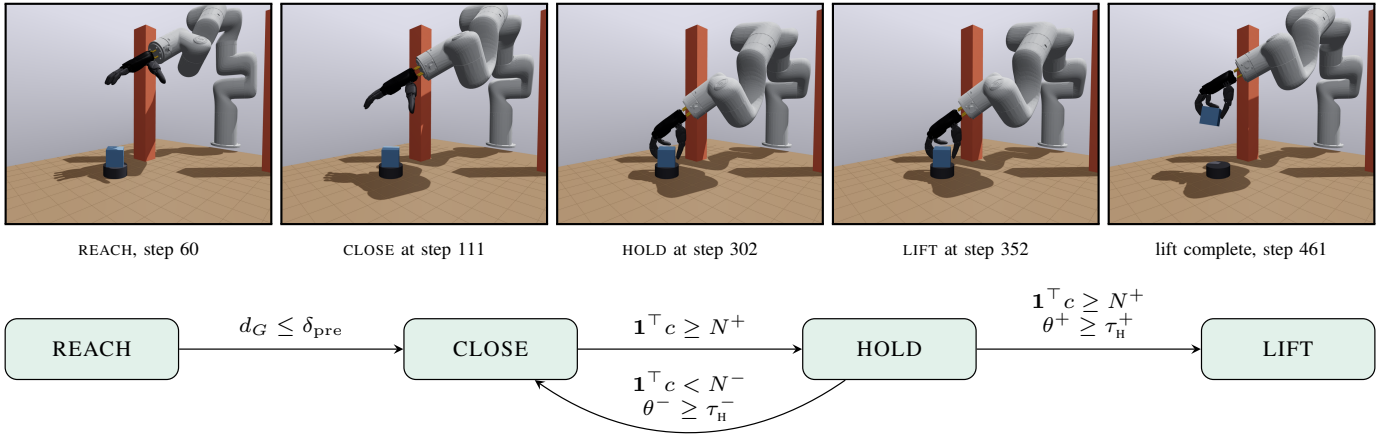

Given the guards, reset maps, and dynamics, the four domains \(D_p\) remain, and we now give the active object-constraint subset \(\Theta_p\) and the nominal command \(v^{\mathrm{nom}}_p\) of each mode. First, in \textsc{reach}, our controller treats the object as an obstacle for the entire hand at full margin, with \(v^{\mathrm{nom}}_{\textsc{r}} = v_{\mathrm{nom}}\) of \eqref{eq:nominal}. The CLF constraint set \(\mathcal R(\textsc{r})\) of \eqref{eq:qp} holds the single constraint \(V = d_G\), and by \cref{cor:safeclf} a reach-time stop appears as slack \(\sigma\) in the trial record.

Next, in \textsc{close}, the coarse finger-link constraints drop out of \(\Theta_{\textsc{c}}\), the palm constraints reduce to a smaller margin, and one clearance constraint per fingertip enters the active constraint set at zero margin. The nominal switches to the closure configuration \(\qgrasp{i^\star}\) of the softmax-dominant candidate, a one-element instance of \cref{def:field}. The stored closure is deliberately deep, every fingertip a full object radius inside the surface. Under each fingertip constraint, the tip stops at the surface while the neighboring fingers continue closing, whereas a closure posed tangent to the surface leaves only one finger touching. Each finger \(f\) also contributes its own CLF constraint, \(V_{\textsc{c},f}(q) = \tfrac12 \lVert q_f - \qgrasp{f,i^\star} \rVert_{\Lambda_f}^2\), with \(\Lambda_f\) the diagonal block of \(\Lambda\) on finger \(f\)'s coordinates. A fingertip clearance constraint that restricts progress toward the stored closure then appears as \(\sigma_f > 0\) in the trial record.

Then, in \textsc{hold}, the arm coordinates are held fixed and each finger satisfies a per-finger hold law,
\begin{equation}
v_{h,f} = \big(1 - c_f\big)\, K_{\textsc{h}} \big(\qgrasp{i^\star} - q\big)_f,
\label{eq:hold}
\end{equation}
which is the solution \eqref{eq:qp} returns for finger \(f\)'s coordinates whenever no fingertip clearance constraint for \(f\) is active, the finger constraint being tight at \eqref{eq:hold} under \(\gamma_r(s) = 2K_{\textsc{h}}\, s\). A touched finger holds its pose and a finger not yet in contact continues closing until its own constraint stops it. Regulating the full hand jointly instead pulls touched fingers off the object wherever the prescribed shape points past it, which induces repeated switching between \textsc{hold} and \textsc{close}. Once the system enters \textsc{hold}, the wrench-quality constraint \(h_{\mathrm{wq}}\) of \eqref{eq:qp} becomes active as a hard constraint alongside self-collision and workspace, fixing \(q_0\) at the moment of entry. By \cref{prop:qualityinv}, the realized margin then stays within \(k_{\mathrm{wq}}\) of its value at hold onset.

Finally, in \textsc{lift}, our controller holds the object fixed in the palm frame and the arm tracks a vertical task-space line through the same program \eqref{eq:qp}. The fingers hold their posture, \(h_{\mathrm{wq}}\) stays active to preserve the risk-adjusted margin through the lift, and we reverse the object's own clearance constraints so that they now enforce clearance between the object and the scene planes and obstacles.

In our simulation, the indicator \(c\) comes from a geometric proxy, the clearance between the fingertips and the object surface. A tactile array supplies the same indicator on hardware, and switching sources changes neither \eqref{eq:guard} nor the constraint sets. Self-collision and workspace constraints remain through every mode, a single quadratic program remains active throughout, and only its constraint set, its CLF target, and the activation of \(h_{\mathrm{wq}}\) change.

\begin{remark}[Hysteresis and Duration Limits]
\label{rem:hysteresis}
The band \(N^{-} < N^{+}\) prevents chatter at the contact boundary, and after a drop from \textsc{hold}, the controller repeats closure as an automatic second grasp attempt. We use the duration threshold \(\tau_{\textsc{h}}^{+}\) in place of a velocity-convergence test, which never triggers while fingers that cannot reach the object continue moving slowly toward the closure target. We use the exit threshold \(\tau_{\textsc{h}}^{-}\) in the opposite direction, holding the \textsc{hold} to \textsc{close} return closed across one contact-source change, so that a single incorrect tactile sample cannot trigger the release. Since \textsc{lift} is terminal, the return is the only cycle in \(\Gamma\), and the two duration limits exclude accumulation of transition times.
\end{remark}

\subsection{From the Field to Wrench Certificates}
\label{sec:bridge}

Our field measures distance in joint space, while the force-closure certificates measure robustness in wrench space. We connect the two so that reaching our target set certifies grasp quality, through four links, at admission, at first contact, during the hold, and at evaluation.

\begin{proposition}[Certified Target Proximity]
\label{prop:certified}
Let every grasp descriptor of \(\Gset\) satisfy \(\epsb_i \ge k_{\mathrm{adm}}\). Then for every \(q\) with \(d_G(q) \le \delta_0\), there exists a candidate \(i\) with
\(\lVert q - \qpre{i} \rVert_{\Lambda} \le \delta_0 + \log N / \rho\)
whose descriptor is risk-adjusted force closed with margin at least \(k_{\mathrm{adm}}\).
\end{proposition}

\begin{proof}
By \eqref{eq:sandwich}, \(d_{\min}(q) \le d_G(q) + \log N/\rho \le \delta_0 + \log N/\rho\), the minimum is attained by some \(i\), and admission guarantees its margin.
\end{proof}

Therefore, through \cref{prop:certified}, the guard threshold \(\delta_{\mathrm{pre}}\) of \eqref{eq:guard} bounds the true distance to a candidate whose descriptor holds the admission certificate. The hypothesis of \cref{prop:certified} is a property of the descriptor set alone, and the evaluation of \cref{sec:eval} gives \(\epsb \ge 0\) on 44 of the 50 stored grasps at the prior of \cref{sec:setup}. The bound therefore applies on the certified subset, and we report the six grasps outside it in \cref{sec:eval}.

Next, at first contact, the hold law \eqref{eq:hold} admits contact sets that deviate from the synthesized grasp, with per-finger variation in contact timing and penetration depth. The wrench matrix realized at hold onset therefore differs from the certificate's, the contact gap we measure in \cref{sec:eval}. \Cref{lem:intrinsic} characterizes the deviations that preserve closure. Every realization whose basis wrenches stay within \(\varepsilon\) of their modeled values remains force closed, and through \cref{lem:ordering} a stored min-weight margin implies such a tolerance without computing \(\varepsilon\) itself.

\begin{remark}[Scope of the Transfer]
\label{rem:transfer}
The tolerance of \cref{lem:intrinsic} is stated in the wrench domain, and it does not identify which joint-space perturbations respect the ball. Our controller therefore measures against the tolerance of \cref{lem:intrinsic} without imposing it as a constraint.
\end{remark}

Then, during the hold, we address the limitation of \cref{rem:transfer} through \cref{prop:qualityinv}. Our constraint \(h_{\mathrm{wq}}\) enforces \(\epsb(q) \ge \epsb(q_0) - k_{\mathrm{wq}}\) directly from hold onset, without assuming that the realized deviations stay within the tolerance of \cref{lem:intrinsic}. At the reported \(k_{\mathrm{wq}} = 0.02\), the constraint holds the executed grasp within \(k_{\mathrm{wq}}\) of its hold-onset margin. Admission thus bounds the hold-onset margin from below, and \cref{prop:qualityinv} limits its decay thereafter, jointly for as long as \eqref{eq:qp} remains feasible.

Finally, we evaluate the realized grasp, computing the risk-adjusted margin \(\epsb\) of \cref{def:riskmargin} on the contact set realized at hold onset and reporting the ratio \(r = \epsb_{\mathrm{exec}} / \epsb_{\mathrm{desc}}\) against the stored descriptor value, with the min-weight pair \(\lbar_{\mathrm{desc}}, \lbar_{\mathrm{exec}}\) reported as a baseline. The risk-adjusted margin certifies force closure with probability at least \(\beta\), the min-weight baseline only at a single nominal friction. A grasp the baseline certifies can still have \(\epsb < 0\) and lose closure at the risk-adjusted friction, and we measure this separation in \cref{sec:eval}.

\section{Experiments}
\label{sec:experiments}

We establish six claims in this section. Our controller reaches, grasps, and lifts across the full descriptor set while retaining most of the synthesized quality margin (\cref{sec:eval}). Our wrench-quality constraint eliminates grasps that the quality-neutral program lifts without a certificate, and a physics execution distinguishes the two outcomes. The same law moves through clutter that admits no unobstructed approach, and stops short of the object without contact when every route closes. Omitting the convergence constraints leaves the numbers of lifts unchanged, and their contribution is instead the recorded slack (\cref{sec:ablate}). Our controller handles a moving obstacle with no new constraints, and we measure the time-derivative term \eqref{eq:qp} does not include (\cref{sec:dynamic}). The same controller runs on a Unitree G1 humanoid without modification (\cref{sec:g1}). We give the platform and the parameters in \cref{sec:setup} and report the numerical checks on the field, gradient, and feasibility claims in \cref{sec:verification}.

\subsection{Robotic Platforms \& Task Setup}
\label{sec:setup}
We instantiate our controller on a fixed-base arm-hand system, a 7-degree-of-freedom arm with an 11-joint multifingered hand and tactile sensors at the fingertips, whose rigid-body model we build in Pinocchio from the URDF, with \(n_a = 7\) and \(n_h = 11\). The thumb contributes one dependent-joint pair at multiplier 1.83, and each outer finger contributes one pair at multiplier 0.89. Five such kinematically coupled pairs appear in \eqref{eq:qp} through \(A_m\), obtained from the rigid-body model at initialization.

In practice, exact mesh distance over the 134 self-collision pairs costs \(330\,\mathrm{ms}\) per query, and our controller therefore replaces every collision mesh with its convex hull at load time, bringing the query to \(0.1\,\mathrm{ms}\). The substitution is conservative, since a hull contains its mesh. The quadratic program solves in \(0.09\,\mathrm{ms}\) per step over 18 variables and roughly two hundred constraints, and the full step including all collision distances costs \(2\,\mathrm{ms}\) against the \(20\,\mathrm{ms}\) control step. \Cref{tab:params} collects every parameter.

In addition, an admission test eliminates every candidate whose pregrasp or closure violates a barrier constraint, obstacle constraints included, which implements the restriction of \cref{prop:certified}. In \cref{sec:discussion}, we measure the trapped equilibrium that results from disabling this test. Beyond the admission test, we evaluate the risk-adjusted margin \(\epsb\) of \cref{def:riskmargin} at confidence \(\beta = 0.9\) under a Gaussian friction prior \(\mu \sim \mathcal{N}(0.70, 0.10^2)\), whose risk-adjusted friction is \(v_\beta = \mathrm{CVaR}_{0.9}(\mu) = 0.524\). Synthesis and execution share \(v_\beta\). This prior follows the construction of \cite{enwerem2026firmgrasp} with no calibration to the 11-joint hand, and the probabilistic closure guarantee of \cref{def:riskmargin} therefore holds to the extent that the model applies.

\subsection{Certified Grasp Descriptor Evaluation}
\label{sec:eval}
Each descriptor passes the synthesis test of \cref{sec:prelim-wrench} before each trial, and we evaluate every executed grasp under \(\epsb\) at the prior of \cref{sec:setup}. In a reference trial, we place a sphere from the descriptor set behind a pair of \(6 \times 6 \times 50\,\mathrm{cm}\) rectangular obstacles, the obstacle-pair scene of \cref{fig:teaser}(a). No unobstructed approach to the object exists. Our filtered closed loop avoids the obstacle, with the palm up to \(20.5\,\mathrm{cm}\) from the nominal path, and completes the full \(12\,\mathrm{cm}\) rise. Both obstacle constraints stay positive at every step including the start pose, with \(\min_{j,t} h_j(q(t)) = 6.7\,\mathrm{mm}\). The hand realizes the descriptor's stored finger set, and the executed grasp retains \(r = 0.912\) of its certified margin, \(3.05\times10^{-3}\) against the descriptor's \(3.35\times10^{-3}\), with the min-weight baseline at \(0.446\) against \(0.468\).
\subsubsection{Descriptor-Set Outcomes} Our controller completes the full reach-grasp-lift sequence on 46 of the 50 objects, each behind a rectangular obstacle placed next to its approach path. \Cref{tab:eval} reports the result on every object, across four primitives, seventeen YCB household objects \cite{calli2015ycb}, and 29 adversarial EGAD meshes \cite{morrison2020egad}. We validate every scene against the arm's own collision hull at the start pose, and no trial begins with the obstacle already inside the obstacle margin. Our controller stops short of the lift on the remaining four trials, in each case at a guard of \eqref{eq:guard} the closed loop does not satisfy. The program \eqref{eq:qp} stays feasible at every step of every trial (\cref{sec:verification}), \cref{thm:safe} therefore holds the safe set forward invariant through the four incomplete trials as well, and the closed loop comes to rest inside the safe set on each of them, the outcome \cref{cor:safeclf} admits when the barrier constraints restrict the soft convergence constraints. In two trials, the hand does not attain the three contacts \eqref{eq:guard} requires within the 700-step horizon, and the trial ends in \textsc{close}. In the other two, our controller reaches \textsc{hold} and terminates at this state without satisfying the lift guard.

The executed grasps also retain their certified quality. Because the hand still realizes a contact set in the two \textsc{hold} failures, we evaluate 48 grasps in all, and 37 of them satisfy \(\epsb \ge 0\) at \(\beta = 0.9\). The min-weight baseline certifies 39 force closed over the same set, with 30 above the \(k_\ell = 0.3\) threshold of \cite{liFroggerFastRobust2023}. The two certificates differ on only two grasps, and both lie within \(4\times10^{-4}\) of zero on either margin.

However, we conduct the descriptor-set evaluation of \cref{tab:eval} without enforcing our wrench-quality constraint, and the \(\epsb\) columns of \cref{tab:eval} therefore evaluate our controller's grasps post hoc. The reference trial that opens \cref{sec:eval} enforces it instead, on the min-weight margin at a fixed zero threshold, and \(h_{\mathrm{wq}} \ge 0\) holds on every control step from the entry into \textsc{hold} through \textsc{lift}. \Cref{prop:qualityinv} states the hold-relative form of the quality guarantee for \(\epsb\), and the reference trial is its measured instance at a fixed threshold.

\subsubsection{Quality-Constrained Evaluation} Our wrench-quality constraint eliminates grasps the quality-neutral controller would lift without a certificate. A second evaluation over the full descriptor set enforces our hold-relative constraint (\cref{prop:qualityinv}) on \(\epsb\) at the prior of \cref{sec:setup}, with every other parameter fixed. Our constraint becomes active only at the entry into \textsc{hold}, and the controller is deterministic. Under both configurations, the controller reaches \textsc{hold} identically on every record.

Our controller completes the lift on 45 objects, one fewer than the quality-neutral configuration, and the four incomplete trials repeat identically, the same two ending in \textsc{close} and the same two in \textsc{hold}. The single trial that changes is a box primitive, where our controller switches the nominal command to the stored closure and the hand forms a three-finger contact set, whose realized margin falls from \(1.51\times10^{-3}\) at hold onset to \(-0.32\), against a descriptor certified at \(1.83\times10^{-3}\). The quality-neutral controller proceeds through the drop and lifts the grasp without a certificate. Our constraint becomes active with \(h_{\mathrm{wq}} = k_{\mathrm{wq}}\) at the entry into \textsc{hold}, the program turns infeasible 25 control steps later, and the trial stops in \textsc{hold}. The same descriptor at a second placement repeats the result, where the closure again realizes only three fingers, and \cref{fig:bslnvsgdf}(b) and \cref{fig:epsbtrace}(b) report the second-placement trial.

Beyond the kinematic evaluation, physics simulation distinguishes the two outcomes. We transfer each held configuration at the entry into \textsc{lift}, the hand-root pose together with the 11 hand joint values, into a Drake simulation with implicit proportional-derivative finger actuation and a fixed wrist, and we increase an upward force on the object to \(1.5\,mg\) over one second. At every commanded finger closure offset from \(0.02\) to \(0.04\,\mathrm{rad}\), the certified grasp of the completed tabletop trial holds the object through the full force increase with at most \(1.5\,\mathrm{mm}\) of lateral drift. Under the three-finger configuration without a certificate, the object falls out of the hand under gravity alone, before we apply the upward force.

In addition, our constraint holds the decay bound across the completed lifts. We evaluate \(\epsb\) at every control step of every recorded trial, validated per trial against the stored \(\epsb_{\mathrm{exec}}\) of \cref{sec:bridge}. At the end of all 45 lifts, the executed margin lies within the \(k_{\mathrm{wq}} = 0.02\) tolerance of its hold-onset value, and satisfies the bound at every evaluated control step on 44. The single exception falls to \(0.060\) below its hold-onset value on a contact-set change and recovers to \(0.006\) below by the end of the lift, against \(0.080\) under the quality-neutral configuration. The median end-of-lift ratio barely changes between the configurations, \(0.932\) quality-constrained against \(0.936\) quality-neutral over the lifts with a positive stored margin. Our constraint therefore leaves the median lift unchanged and binds only on the single exception, where the margin decay approaches the \(k_{\mathrm{wq}} = 0.02\) tolerance.

Finally, when we evaluate the stored descriptors themselves, 44 of 50 satisfy \(\epsb \ge 0\) and 46 of 50 single-friction force closure. The executed-margin gaps in \cref{tab:eval} therefore reflect the switch to the stored closure and the realized contact set on top of the six stored descriptors that do not themselves satisfy the risk-adjusted certificate. \cref{fig:evalscatter} plots the ratio \(r\) over the full descriptor set, \cref{fig:epsbtrace} shows the realized margin through the hold and the lift for the completed tabletop trial and the rejected box primitive, and \cref{fig:fieldbox} renders a single sphere trial in full.

\subsubsection{Obstacle-Dense Tabletop Scenes} Our controller passes through a gap with no unobstructed approach and, when clutter closes every route, stops short of the object without contacting an obstacle. We evaluate a scene pair under identical parameters, with the model joint-speed bounds scaled by 0.25, a step-size condition we return to below. Both scenes run with the wrench constraint \(h_{\mathrm{wq}}\) inactive.

In the first scene, we place the target object between the two obstacles, the trial the tabletop panels of \cref{fig:teaser}(a) depict. Our controller moves through the gap and completes the \(12\,\mathrm{cm}\) rise on the mode schedule \cref{fig:modelogic} plots. \(h_{\mathrm{obs}}\) stays positive throughout, reaching a minimum of \(+17.4\,\mathrm{mm}\) with \(3.2\,\mathrm{cm}\) of true surface clearance. For the realized grasp, \(\epsb_{\mathrm{exec}} = 3.80\times10^{-3}\) under the post-hoc evaluation of \cref{sec:bridge} and \(\lbar_{\mathrm{exec}} = 0.629\) under the baseline.

Next, in the second scene, we add a third obstacle on the segment from the hand to the target object. Our controller never enters \textsc{close} there, and the field value remains at \(0.524\) against the \(0.12\) guard threshold through 2400 control steps while the closed-loop trajectory stops at the added obstacle. We measure \(h_{\mathrm{obs}} < 0\) from step 68 on 2286 of those 2400 steps, with \(\min_t h_{\mathrm{obs}}(q(t)) = -13.0\,\mathrm{mm}\) against the \(15\,\mathrm{mm}\) margin and \(2.0\,\mathrm{mm}\) of true surface clearance remaining. The closed-loop trajectory comes to rest off the obstacle surface but inside the certified obstacle margin.

In the second scene at the full model speed bounds, the mean reach velocity reaches \(5.7\,\mathrm{rad/s}\), and the segment between two consecutive \(20\,\mathrm{ms}\) steps passes through an obstacle that the two sampled configurations clear, with the true penetration reaching \(4.5\,\mathrm{cm}\). Scaling the bounds by 0.25 eliminates the penetration and recovers the first-scene result. \Cref{thm:safe} therefore holds for the continuous-time closed loop, and the sampled implementation inherits it only under a step-size condition that our other reported trials satisfy.

More broadly, over the descriptor set, the arm holds every obstacle and self-collision constraint clear of the true surface in all 44 completed trials, and enters the obstacle margin in seven, at most by \(6.3\,\mathrm{mm}\) of the \(15\,\mathrm{mm}\) obstacle margin with \(8.7\,\mathrm{mm}\) of surface clearance still remaining. Sampling recovers margin more slowly than \eqref{eq:hbound} does in continuous time, and we report the resulting violation as the cost of the discretization. We exclude the fingertip constraints by construction, since the mode switch of \cref{sec:modes} drives them negative to admit contact.

\begin{table}[t]
\centering
\caption{Certified-Descriptor Evaluation by Object Class. Shading marks the adversarial class.}
\label{tab:eval}
\footnotesize
\setlength{\tabcolsep}{3pt}
\begin{tabular}{@{}lrrrrr@{}}
\toprule
Class & Objects & Lift & \(\epsb \ge 0\) & Force closed & \(\lbar \ge 0.3\)\\
\midrule
Primitives & 4 & 4 & 3 & 3 & 2\\
YCB household & 17 & 14 & 12 & 12 & 10\\
\rowcolor{oursmint}
EGAD adversarial & 29 & 28 & 22 & 24 & 18\\
\midrule
\rowcolor{rowtot}
All & 50 & 46 & 37 & 39 & 30\\
\bottomrule
\end{tabular}
\end{table}

\begin{figure}[t]
\centering
\includegraphics[width=.8\columnwidth]{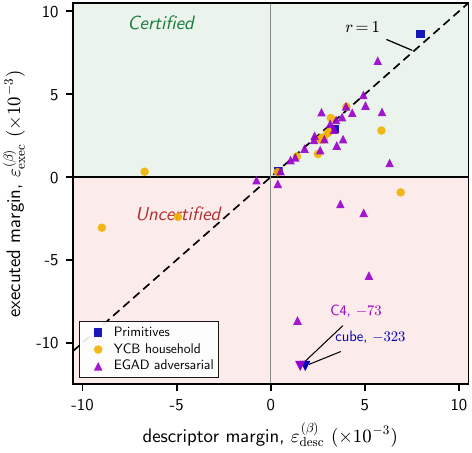}
\caption{\textbf{The Quality Ratio Across the Descriptor Set.} We plot the executed risk-adjusted margin against the stored descriptor margin for every evaluated grasp, colored by class as in \cref{tab:eval}, with the diagonal \(r = 1\). Two grasps fall below the axis range, and we draw them at the bottom edge with their executed margins \(\epsb_{\mathrm{exec}}\) labeled next to them.}
\label{fig:evalscatter}
\end{figure}

\begin{figure*}[t]
\centering
\includegraphics[width=\textwidth]{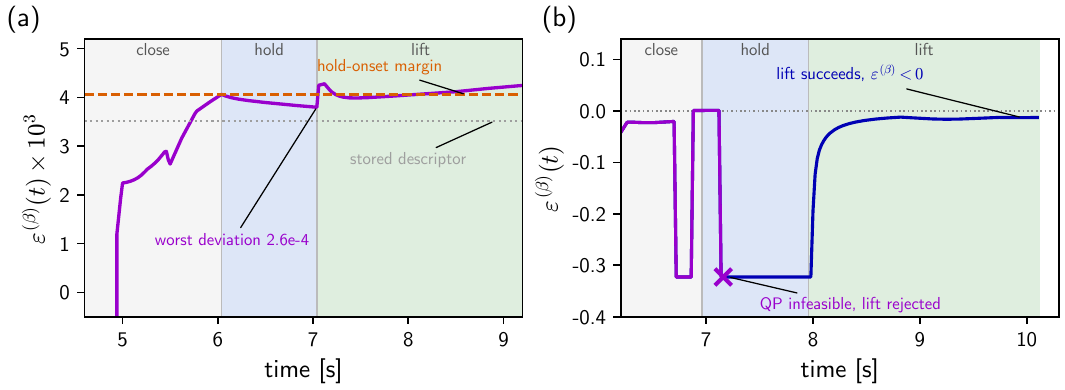}
\caption{\textbf{Per-Step Margins from the Offline Evaluation.} We plot the realized risk-adjusted margin per control step for two executions, with the mode bands shaded. (a) On the cluttered tabletop scene, the margin enters \textsc{hold} at \(4.06\times10^{-3}\) and never falls below \(3.80\times10^{-3}\) at every control step through the lift, a worst deviation of \(2.6\times10^{-4}\), 0.013 of the \(k_{\mathrm{wq}} = 0.02\) tolerance of \cref{prop:qualityinv}. (b) On an object behind the column pair of \cref{fig:bslnvsgdf}(b), the realized margin falls during the hold on a contact change. The quality-neutral program proceeds through the drop and completes the lift at \(\epsb = -0.013\), below the closure certificate, while with our constraint enforced, the program turns infeasible at the drop, the transition to \textsc{lift} never occurs, and the trial stops in \textsc{hold}.}
\label{fig:epsbtrace}
\end{figure*}

\subsection{Omitting the Convergence Constraints}
\label{sec:ablate}

Omitting the convergence constraints of \eqref{eq:qp} and re-running the full descriptor set leaves the lift numbers unchanged, 46 either way, and \cref{tab:ablate} reports both configurations. The certified counts move from 37 to 38, and the mean executed \(\epsb\) from \(-6.7\times10^{-3}\) to \(1.2\times10^{-3}\) over the 48 evaluated trials. Three trials change status between the configurations, one completing only with the constraints present, one only with them omitted, and one reaching \textsc{hold} under both configurations but falling short of the lift only without them. On the descriptor set, the convergence constraints neither add lifts nor reduce them beyond the spread a single contact difference produces.

By contrast, the convergence constraints contribute the recorded slack. We read \(\sigma > 0\) from the field constraint on all 46 trials, with a mean of \(1.23\) over \textsc{reach}, and the per-finger slack variables average \(0.79\) over \textsc{close}. The barrier constraints are thus active against the nominal on every trial in the descriptor set. When a trial stops, the slack variables show which constraint failed to converge and by how much, where a controller without them would stop with no such record, and \cref{cor:safeclf} confirms that this slack never relaxes a barrier constraint, leaving it a measure of impeded progress alone.

\begin{table}[t]
\centering
\caption{Comparison with the Convergence Constraints Omitted}
\label{tab:ablate}
\footnotesize
\setlength{\tabcolsep}{1.0pt}
\renewcommand{\arraystretch}{1.05}
\begin{tabular}{@{}lrrrrrr@{}}
\toprule
Setting & Lift & \(\epsb \ge 0\) & Mean \(\epsb\) & \(\lbar \ge 0\) & \(\lbar \ge 0.3\) & Mean \(\lbar\)\\
\midrule
\rowcolor{oursmint}
Full \eqref{eq:qp} & 46 & 37 & \cellcolor{losssalmon}-0.007 & 39 & 30 & 0.161\\
Constraints omitted & 46 & 38 & 0.001 & 41 & 29 & 0.234\\
\bottomrule
\end{tabular}
\end{table}

\begin{figure*}[t]
\centering
\includegraphics[width=\textwidth]{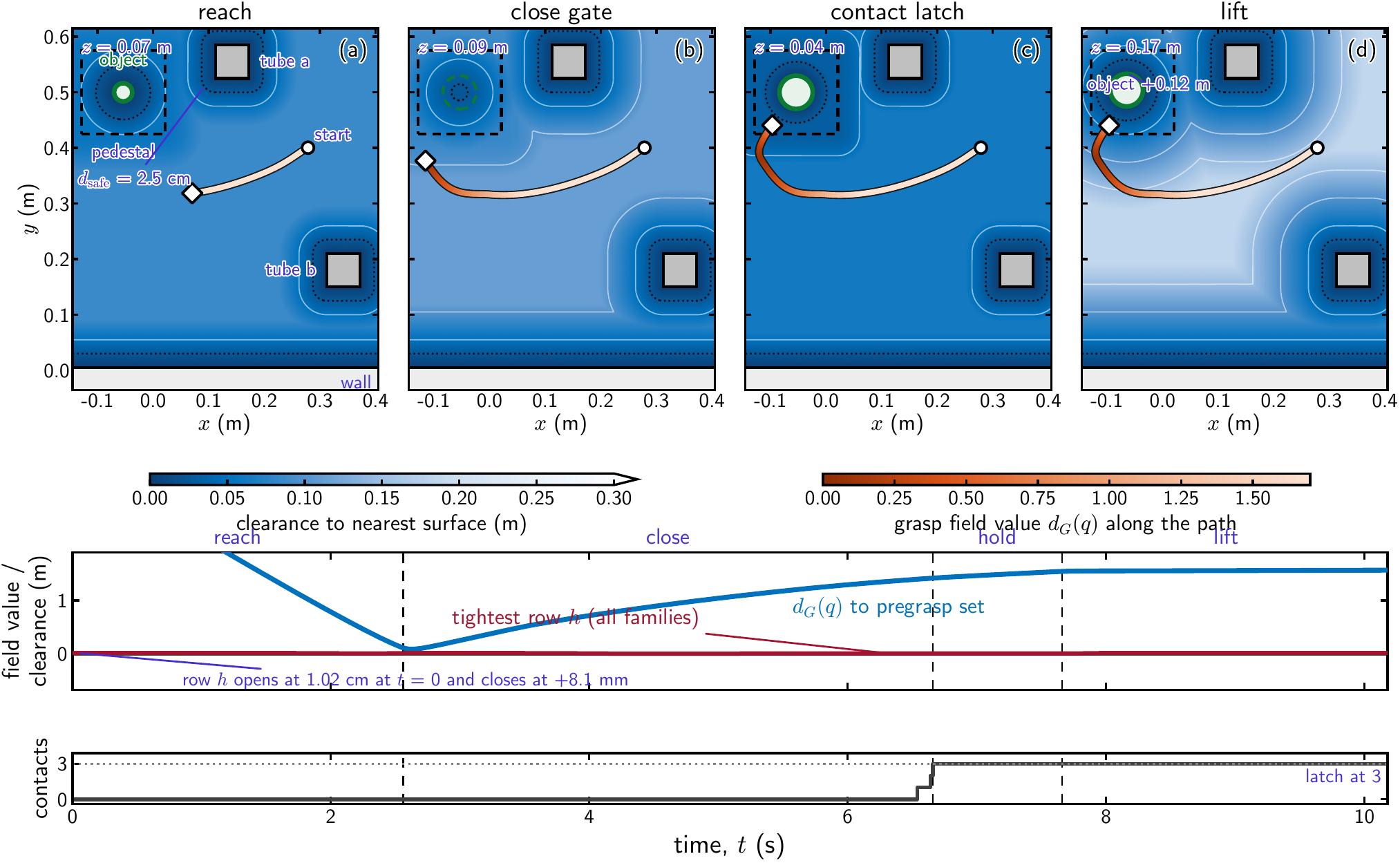}
\caption{\textbf{Field Execution Through Contact and Lift.} We plot the clearance field that the obstacle, wall, and object constraints induce over the workspace at the four mode transitions of a two-obstacle trial, with the executed palm path colored by our grasp field value \(d_G\). The band below plots the per-step series, the field decrease through \textsc{reach}, the rise after the switch to the stored closure through \textsc{close}, the number of contacts with its return to \textsc{close}, and the minimum barrier constraint together with the obstacle constraints, nonnegative throughout apart from the single home-posture hull-overlap pair that the filter drives positive per \cref{thm:safe}.}
\label{fig:fieldbox}
\end{figure*}

\subsection{Dynamic Obstacles}
\label{sec:dynamic}

Our controller updates the obstacle poses at every step and therefore handles a moving obstacle without modification. We place a \(5 \times 5 \times 45\,\mathrm{cm}\) obstacle at the end of the arm's approach path and translate it at \(0.05\,\mathrm{m/s}\) across this path during the arm's approach motion. Our filtered closed loop avoids the moving obstacle and completes, entering \textsc{close} at step 159 against 41 without the obstacle, and completing the \(12\,\mathrm{cm}\) rise at step 428. The filtered palm path lies up to \(19.1\,\mathrm{cm}\) from the nominal path. The filtered closed loop thus drives the arm to evade the moving obstacle and continue its approach motion as the obstacle crosses the approach path. Surface clearance stays positive throughout, with a minimum of \(+8.4\,\mathrm{mm}\).

For the same moving-obstacle case, \eqref{eq:qp} omits the \(\partial h/\partial t\) term, and \(h_{\mathrm{obs}}\) therefore falls below zero by at most the approach speed divided by the barrier rate, \(v/\alpha_0 = 10\,\mathrm{mm}\) here. We measure a minimum of \(-6.6\,\mathrm{mm}\), inside the \(1.5\,\mathrm{cm}\) obstacle margin. However, one placement we test reproduces the trapped equilibrium of \cref{sec:discussion}. The obstacle reaches the approach line before the palm, the filtered closed loop stops behind it at zero velocity, and stays there after the obstacle comes to rest.

\subsection{Humanoid Instantiation}\label{sec:g1}%
Our controller runs on a humanoid without modification, which tests that the platform-specific structure comes from the rigid-body model alone. We merge the 29-degree-of-freedom Unitree G1 body with the same 11-joint hand at the right wrist in Pinocchio, and we fix the legs and left arm, leaving a 21-dimensional fixed-pelvis chain of three torso and seven arm joints plus the hand. We define the tabletop half-space in the pelvis frame and load a grasp record synthesized for the humanoid through the construction of \cref{def:targetset}.

Specifically, we place a \(4 \times 4 \times 36\,\mathrm{cm}\) obstacle on the line from the start palm to the object. Our controller completes the \(12\,\mathrm{cm}\) rise in 434 steps with no infeasible step. The minimum value of the obstacle constraint is \(3.8\,\mathrm{mm}\) with \(3.4\,\mathrm{cm}\) of true surface clearance, the filtered palm path lies up to \(9.7\,\mathrm{cm}\) from the nominal path, and the torso rotation stays under \(18^\circ\) through the lift.

However, the same contact gap we quantify in \cref{sec:eval} appears at hold onset. The guard is satisfied with three fingers in contact before the thumb arrives, and both the risk-adjusted margin and the min-weight baseline place the executed three-contact set below closure, against a descriptor whose four-contact closure has \(\epsb = 4.31\times10^{-3}\) and a min-weight margin of \(0.48\).

\subsection{Verification}
\label{sec:verification}

We verify the field, gradient, and feasibility claims on the full rigid-body model, over five \firmgrasp{} descriptors at smoothing parameter \(\rho = 25\) and 200 sampled configurations. The analytic gradient \eqref{eq:grad} matches central finite differences to a maximum error of \(7.7 \times 10^{-10}\), which is the finite-difference noise floor at a step of \(10^{-6}\). The gradient bound of \cref{prop:field} holds at a maximum \(\Lambda^{-1}\)-norm of \(1 + 2.2\times10^{-16}\), unity to one unit in the last place.

Similarly, the bound \eqref{eq:sandwich} holds on both sides, with no violation on the lower side and an upper gap of at most \(3.12 \times 10^{-2}\) against the ceiling \(\log 5/\rho = 6.44 \times 10^{-2}\). The 200 sampled configurations therefore exercise the bound where it is nearly tight, at the Voronoi boundaries between distinct descriptors.

Next, we run a free-space execution over the same five candidates through every mode. The controller enters \textsc{close} at step 52, \textsc{hold} at step 184, and \textsc{lift} at step 234, and completes the \(12.1\,\mathrm{cm}\) rise in 337 steps, with the gap of \eqref{eq:sandwich} at most \(1.93 \times 10^{-2}\) along the trajectory. The recorded per-step softmax weights concentrate on candidate 4, the nearest candidate in joint space from the start posture, and remain there through closure and the hold regardless of which candidate has the largest margin. The trial is our first empirical instance of the no-selection claim, with the executed candidate coming from the field's own weights and no separate selection step. For the realized four-finger grasp, \(\epsb = +2.27 \times 10^{-3}\) at a min-weight margin of \(+0.346\).

Finally, the program \eqref{eq:qp} stays feasible at every step of every reported trial, and the dependent-joint equality residual reaches at most \(1.4 \times 10^{-15}\). Every CBF family holds nonnegative after the initial steps, apart from a single home-posture self-collision pair that the hull substitution starts at \(-1.65\,\mathrm{cm}\) while the underlying meshes remain clear. The pair recovers to \(+1.4\,\mathrm{mm}\) by the end of the trials, as \eqref{eq:hbound} predicts.

\begin{table}[t]
\centering
\caption{Controller Parameters for the Verification Trials}
\footnotesize
\setlength{\tabcolsep}{1.4pt}
\label{tab:params}
\begin{tabular}{@{}llr@{}}
\toprule
Symbol & Role & Value\\
\midrule
\(\rho\) & softmin smoothing parameter, \eqref{eq:field} & 25\\
\(k\) & nominal gain, \eqref{eq:nominal} & 2.0\\
\(\Lambda\) & arm / hand metric weights & 1.0 / 0.35\\
\(\alpha_0\) & barrier rate, \eqref{eq:qp} & 5.0\\
\(\eta\) & CLF slack weight, \eqref{eq:qp} & \(10^{3}\)\\
\(K_{\textsc{h}}\) & hold gain, \eqref{eq:hold} & 0.6\\
\(k_{\mathrm{adm}}\) & \(\epsb\) admission threshold at synthesis & 0\\
\(k_{\mathrm{wq}}\) & wrench-constraint decay tolerance & 0.02\\
& safety margin, all clearance constraints & 2.5\,cm\\
& palm margin in \textsc{close} & 0.5\,cm\\
& fingertip radius & 0.3\,cm\\
& contact threshold on tip clearance & 0.6\,cm\\
\(\delta_{\mathrm{pre}}\) & \textsc{reach} to \textsc{close} guard & 0.12\\
\(N^{+}/N^{-}\) & contact limits, hold / release & 3 / 2\\
& obstacle margin & 1.5\,cm\\
\(\tau_{\textsc{h}}^{+}\) & hold duration to enter \textsc{lift} & 1\,s\\
\(\tau_{\textsc{h}}^{-}\) & exit duration to return to \textsc{close} & 60\,ms\\
& lift speed / height & 5\,cm/s / 12\,cm\\
& control step & 20\,ms\\
\bottomrule
\end{tabular}
\end{table}

\section{Discussion}
\label{sec:discussion}

The closed-loop trajectory of our CBF quadratic program is a collision-free curve in \(\Cfg\), ending at contact admission where a conventional plan-then-track system completes its planned reach and begins closure. We add our wrench-quality CBF at hold onset, and \cref{prop:qualityinv} bounds the margin's decay. Barrier constraints tuned for slip and roll-off \cite{shawcortez2019cbf}, or risk-sensitive constraints over latent uncertainty \cite{enweremRiskConstrainedBeliefSpaceOptimization2026a}, enter \eqref{eq:qp} through the same hard-constraint block and leave \cref{thm:safe} unchanged.

\subsection{Trapped Equilibria of the Filtered Closed Loop}
As with artificial potential fields, nonconvex constraints can induce undesired equilibria \cite{khatib1986realtime}. Such an equilibrium occurs when \eqref{eq:qp} removes every descent direction and the projected gradient has no tangential component (\cref{fig:teaser}(c)). Multiple candidates enlarge the attraction region, and admission removes candidates blocked at their endpoints, but neither mechanism establishes global reachability. Without admission, the five-candidate experiment stalls above the guard threshold; with admission, it removes the blocked candidate and reaches \textsc{close}. A placement sweep fails for palm-obstacle overlap above \(7\,\mathrm{cm}\) and succeeds at \(3\,\mathrm{cm}\). Tangential nominal terms or perception resets can mitigate this failure without changing \cref{thm:safe} \cite{singletary2021comparative}. 

\subsection{Limitations}
Our field depends on candidates synthesized for the current object pose, and a moving object requires the candidate generator to update \(\Gset\) at every perception step. On hardware a pose estimate replaces ground truth, where measurement-robust constraints incorporate bounded state error into the CBF margin at a known cost in conservatism \cite{cosner2021mrcbf}. The verification reported here is kinematic, the executed grasps hold a quasi-static closure certificate while the lift moves an attached object, and we do not test on hardware. The tactile contact source of \cref{sec:modes} defines the interface such trials would use, filters of the same form already operate at every control step on dynamic legged platforms \cite{molnar2022modelfree,bena2025geometry,yang2025cbfrl}, and the open question under dynamics concerns the wrench certificate alone.

\section{Conclusion}
\label{sec:conclusion}

We formulated GDF feedback for trajectory-free multifingered grasp execution. The softmin field interpolates candidate selection; the CBF-QP establishes forward invariance; hysteretic transitions admit contact; and the wrench-quality barrier bounds margin loss after hold onset. The 18-DoF implementation requires 2 ms per 20 ms interval and transfers directly to the Unitree G1 model.

\bibliographystyle{IEEEtran}
\hbadness=3500
\bibliography{refs}

\end{document}